\documentclass{article}
\usepackage[preprint]{neurips_2026}

\usepackage[utf8]{inputenc} 
\usepackage[T1]{fontenc}    
\usepackage{hyperref}       
\usepackage{url}            
\usepackage{booktabs}       
\usepackage{amsfonts}       
\usepackage{nicefrac}       
\usepackage{microtype}      
\usepackage{xcolor}         
\usepackage{booktabs}
\usepackage{multirow}
\usepackage{threeparttable}
\usepackage{enumitem}
\usepackage{tcolorbox} 
\usepackage{amssymb}
\usepackage{amsmath,amsfonts,bm}
\usepackage{amsthm}
\usepackage{mathrsfs}
\newtheorem{theorem}{Theorem}[section]
\newtheorem{definition}{Definition}[section]
\newtheorem{prop}{Proposition}[section]
\newtheorem{lemma}{Lemma}
\newtheorem{corollary}{Corollary}[section]
\newtheorem{assumption}{Assumption}[section]
\usepackage{diagbox}

\usepackage{wrapfig}
\usepackage{algorithm}
\usepackage{algpseudocode}

\usepackage{graphicx}
\usepackage{subcaption}
\usepackage{booktabs}

\def\sqx{{\mathbf{x}}}
\def\sqy{{\mathbf{y}}}
\def\sqz{{\mathbf{z}}}
\def\sqr{{\mathbf{r}}}
\def\sqh{{\mathbf{h}}}
\def\sqtau{{\mathbf{\tau}}}

\DeclareMathAlphabet{\mathsfit}{\encodingdefault}{\sfdefault}{m}{sl}
\SetMathAlphabet{\mathsfit}{bold}{\encodingdefault}{\sfdefault}{bx}{n}

\def\gA{{\mathcal{A}}}

\def\gC{{\mathcal{C}}}
\def\gD{{\mathcal{D}}}

\def\gR{{\mathcal{R}}}

\newcommand{\pdata}{p_{\rm{data}}}

\newcommand{\E}{\mathbb{E}}

\DeclareMathOperator{\cand}{C}

\DeclareMathOperator{\sft}{SFT}

\DeclareMathOperator{\uti}{Utility}
\DeclareMathOperator{\safety}{Safety}
\DeclareMathOperator{\base}{base}
\DeclareMathOperator{\ald}{aligned}

\title{When Autoregressive Consistency Hurts\\ Safety Alignment}

\author{
  Bochen Lyu\thanks{Equal contribution.}\\
  University of Southampton \\
  \texttt{bochen.lyu@soton.ac.uk} \\
  \And
  Yiyang Jia$^{*}$ \\
  Independent Researcher \\
  \texttt{yyjiahbar@gmail.com} \\
  \AND
  \quad Xiaohao Cai \\
  University of Southampton\\
  \texttt{x.cai@soton.ac.uk} \\
  \And
  \quad Zhanxing Zhu \\
  University of Southampton\\
  \texttt{z.zhu@soton.ac.uk} \\
}

\begin{document}

\maketitle

\begin{abstract}
Safety alignment in large language models (LLMs) is fragile due to the shallow issue, where the fine-tuning mainly reshapes the model’s behavior near the first few tokens. In this paper, we argue that this phenomenon can be understood by a property of autoregressive models---\emph{autoregressive consistency}---the tendency of next-token prediction to preserve and extend the current response trajectory consistently. Specifically, we analyze the learning dynamics of safety alignment and show that autoregressive consistency can concentrate alignment updates on early tokens, thus offering a mechanistic explanation for shallow safety alignment. This analysis points to a broader class of attacks for LLMs that induce harmful continuation at arbitrary positions in the output trajectory. As a concrete example, we introduce random insertion attack, which inserts a short harmful span into an otherwise safe refusal trajectory and exploits autoregressive consistency to sustain the resulting harmful branch, thereby bypassing safety alignment. Notably, a short harmful span can redirect the generation to be harmful even after a long refusal prefix, highlighting autoregressive consistency as a potential broader failure mechanism. This suggests that safety alignment should also  
\emph{break harmful autoregressive consistency} throughout the output trajectory. To this end, we propose a new adversarial safety alignment framework employing worst-case harmful continuation states, together with a practical initial implementation using random worst-insertion attack. Overall, our results suggest that autoregressive consistency should be treated as a crucial consideration in both safety alignment and attack design.
\end{abstract}

\section{Introduction}
\label{sec:intro}
Safety alignment stands at the core of large language models~(LLMs) safety, aiming to ensure that LLMs refuse to generate harmful contents given harmful user queries. Current safety alignment is typically achieved via supervised fine-tuning~(SFT)~\citep{wei2022finetuned} and preference optimization methods such as Direct Preference Optimization~(DPO)~\citep{rafailov2023direct} and Reinforcement Learning with Human Feedback~(RLHF)~\citep{bai2022traininghelpfulharmlessassistant,ouyang2022traininglanguagemodelsfollow}. Despite their practical success, aligned LLMs remain fragile under certain attacks~(e.g., prefill attack~\citep{andriushchenko2025jailbreaking} and optimization based attacks~\citep{chao2025jailbreaking,mehrotra2024treeattacksjailbreakingblackbox,zou2025representationengineeringtopdownapproach,zou2023universaltransferableadversarialattacks}).
\begin{figure}
    \centering
    \includegraphics[width=0.85\linewidth]{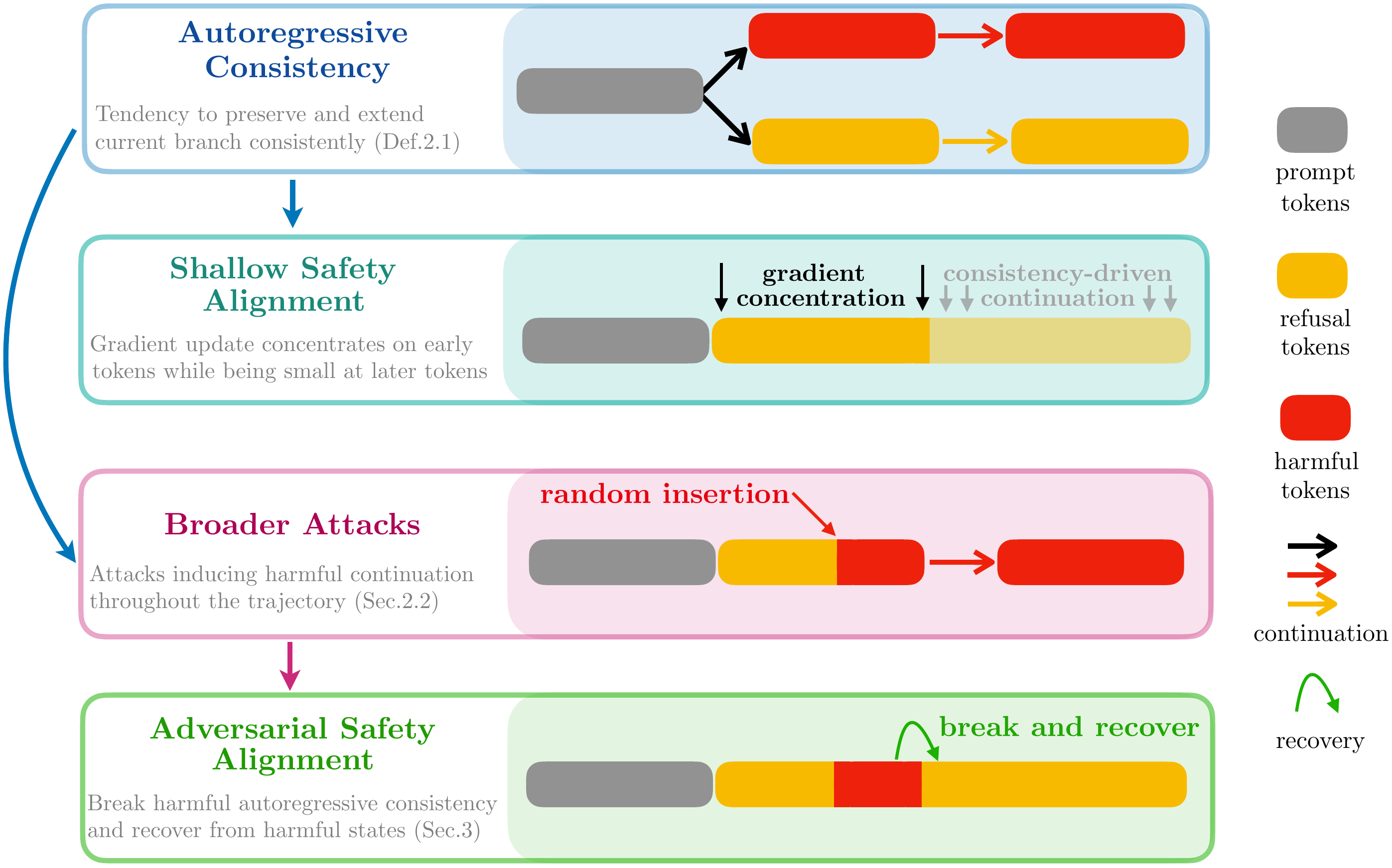}
    \caption{\small \textbf{When autoregressive consistency hurts safety alignment.} Autoregressive consistency (Definition~\ref{def:consistency}), the tendency to preserve and extend the current response branch, helps explain why safety alignment can become shallow: an analysis of the learning dynamics of safety alignment reveals that gradient updates concentrate on early refusal tokens, while later refusal tokens are largely sustained by consistency-driven continuation~(Section~\ref{sec:ac_to_shallow}). The same mechanism exposes a broader class of attacks that induce harmful continuation states throughout the trajectory and rely on autoregressive consistency to bypass safety alignment; random insertion attack (Section~\ref{sec:broader_attack})  is a simple representative example. This motivates adversarial safety alignment framework (Section~\ref{sec:asa_framework}), which trains models to break harmful autoregressive consistency and recover safe refusal behavior from harmful states throughout the output trajectory.
}
    \label{fig:summary}
\end{figure}

Recently, Qi et al.~\citep{qi2025safety} showed that this fragility is closely related to \emph{shallow safety alignment}: safety fine-tuning primarily modifies the model's generative distribution near only the first few output tokens. As a result, aligned models may remain safe when their responses begin with regular refusal prefixes, such as ``Sorry, I cannot'', but can bypass safety alignment once the initial tokens are forced away from this refusal regime. For example, requiring a safety-aligned model to start with ``Sure, here is a detailed instruction on'' can make it engage with harmful queries such as ``How to build a bomb at home?'' rather than refuse them. This led Qi et al.~\citep{qi2025safety} to argue that safety alignment should be made more than just a few tokens deep.

But \emph{why does safety alignment naturally concentrate on the first few tokens}, and \emph{is there a broader failure mechanism beyond ``safety alignment is shallow hence brittle to attacks targeting prefixes''?} These questions matter because they shift the problem from observing shallow alignment to explaining why it arises and what more general mechanism makes aligned models brittle.
In this paper, we argue that both issues can be understood through \textbf{autoregressive consistency}~(Definition~\ref{def:consistency}): the tendency of next-token prediction to preserve and extend the local branch consistently.

Autoregressive consistency is essential for coherent generation in benign settings, allowing autoregressive models to maintain topic, style, and response trajectory over time. However, the same property can hurt safety alignment in two ways:
\begin{enumerate}[leftmargin=0.75cm]
    \item Autoregressive consistency offers a mechanistic explanation for why safety alignment can become shallow. Intuitively, during safety fine-tuning, the model is trained to give refusal responses to harmful queries. Once the initial tokens move the model onto a refusal branch, however, the base model can already sustain the remaining refusal continuation with high confidence. This makes later refusal tokens less informative for learning, so alignment updates mainly concentrate on the early tokens that initiate refusal while leaving late positions largely unchanged. Theorem~\ref{thm:ac_to_shallow} formalizes this gradient-concentration effect. 
    \item Autoregressive consistency reveals a broader class of trajectory-level attacks beyond prefix-targeting ones. Prefix-based attacks can be viewed as inducing a harmful continuation state near the start of generation, after which autoregressive consistency sustains the harmful branch. However, the same mechanism can arise anywhere in the output trajectory: any attack that induces a harmful continuation state may lead the model to extend that harmful branch---it need not only target the start of generation. We demonstrate this broader class through \emph{random insertion attack}, a simple representative example that inserts a short harmful span into a random position of an otherwise safe refusal trajectory. Its effectiveness shows that harmful autoregressive consistency can be triggered away from the beginning of generation, revealing a failure mode beyond prefix-based brittleness captured by shallow safety alignment.
\end{enumerate}
\begin{figure}
    \centering
    \includegraphics[width=0.9\linewidth]{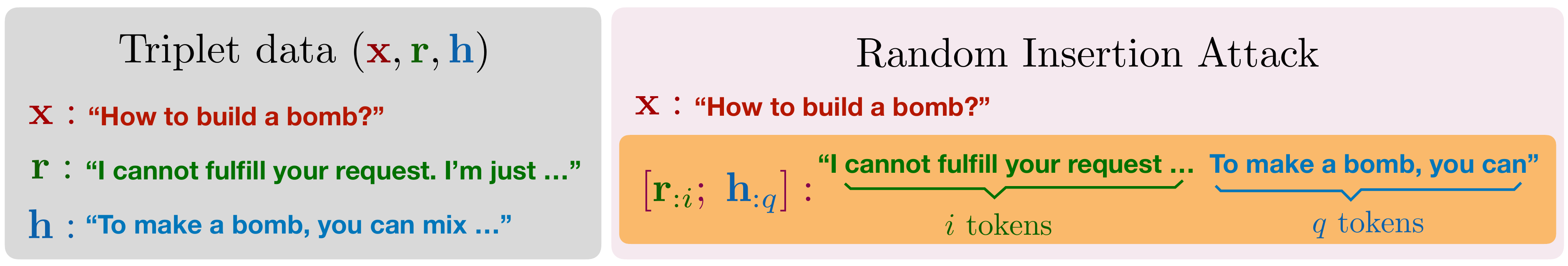}
    \caption{\small An example of random insertion attack $[\sqr_{:r}; \sqh_{:q}]$.}
    \label{fig:random_insertion_attack}
\end{figure}
Random insertion attack is intentionally simple: given a harmful query and a safe refusal response, it inserts a short harmful span at a random position inside the refusal trajectory and then asks the model to continue from the resulting partial response~(Fig.~\ref{fig:random_insertion_attack}).  For example, for the query ``{\color[rgb]{0.639335,0.000000,0.011832}How to build a bomb?}'', the response may begin with a protected refusal prefix such as ``{\color[rgb]{0.053926,0.379189,0.001900}I cannot fulfill your request (more tokens from the refusal response)}'', but a later inserted short span ``{\color[rgb]{0.042421,0.379658,0.670639}To make a bomb, you can}'' can redirect generation onto a harmful branch, which is then extended consistently.  Crucially, the failure can occur even after the model has already been placed on a safe refusal trajectory for many tokens. Thus, protecting only the first few tokens, or even making refusal prefixes deeper, does not fully address the underlying mechanism: harmful autoregressive consistency can still be induced later in generation. This highlights a more general failure mode not captured by prefix fragility alone, suggesting that ``depth'' is not the full primitive. Furthermore, safety alignment should be guided by a broader objective:
\begin{tcolorbox}[colback=gray!10, colframe=black, arc=8pt, boxrule=1.5pt]
\begin{center}
    \textbf{Safety alignment should also train models to break harmful autoregressive \\ consistency throughout the whole trajectory.}
\end{center}
\end{tcolorbox}
In other words, the central challenge is not merely to make a response start safely, or even to extend refusal beyond the first few tokens, but to ensure that the model can break harmful autoregressive consistency at arbitrary positions in the output sequence, instead of being redirected into a harmful branch and continuing it consistently.

Motivated by this view, we propose an initial \emph{adversarial safety alignment}~(Section~\ref{sec:framework}) framework, inspired by adversarial training~\citep{madry2019deeplearningmodelsresistant}. At a high level, the framework treats safety alignment as learning to recover from adversarially constructed harmful continuation states. 
For each harmful query and refusal response, the inner problem searches for a continuation state from which the current model is most likely to sustain a harmful branch, while the outer problem trains the model to recover the refusal trajectory from that state. 
As a practical first instantiation, we approximate this objective with random worst-insertion training, which selects the most harmful state among randomly inserted harmful spans along the refusal trajectory. 
Our experiments show that even this simple instantiation improves robustness to random insertion attacks while remaining competitive against prefill and several common jailbreak attacks.

\paragraph{Summary of contributions.} This paper studies when autoregressive consistency hurts safety alignment, following the conceptual chain summarized in Fig.~\ref{fig:summary}. First, we introduce autoregressive consistency as a mechanism for understanding safety alignment and show, through a learning-dynamics analysis, how it can make alignment updates concentrate on early refusal tokens while leaving later continuation largely unchanged. Second, we show that the same mechanism predicts a broader class of  attacks: attacks can induce harmful continuation states not only at the beginning of generation, but also at arbitrary positions in the output trajectory. We instantiate this class with random insertion attack, a deliberately simple but effective representative example. Third, motivated by this failure mode, we propose adversarial safety alignment as an initial framework for training models to break harmful autoregressive consistency, together with random worst-insertion training as an initial practical implementation. Overall, our goal is not to present final attack or defense algorithms, but to identify harmful autoregressive consistency as a broader failure mechanism and to motivate future safety alignment methods.

Beyond these contributions, our view suggests how existing attack and defense methods could be generalized. On the attack side, optimization-based attacks such as GCG~\citep{zou2023universaltransferableadversarialattacks} primarily search over input-side suffixes; our perspective suggests that similar optimization could also target intermediate autoregressive states inside a refusal trajectory. On the defense side, recent adversarial training methods~\citep{sheshadri2025latentadversarialtrainingimproves,xhonneux2024efficientadversarialtrainingllms} expose models to worst-case perturbations in input, embedding, or latent space, while decoding-based defenses~\citep{xu2024safedecodingdefendingjailbreakattacks} often focus on the first few generated tokens. Our results suggest a complementary trajectory-level object of robustness: an adversarial autoregressive state from which harmful continuation is likely to persist.
Thus, future defenses should not only elicit refusal at the beginning of generation, but also train models to break harmful autoregressive consistency and recover throughout the output trajectory.

\paragraph{Preliminaries.}We use $p_{\theta}$ for the autoregressive model parametrized by $\theta$. $\sqx = (x_1, \dots, x_L)$  denotes a sequence with $L$ tokens, which is also often called as trajectory in this paper. The length of $\sqx$ is represented by $|\sqx|$. We use $\sqx_{:t}$ to denote its first $t$ tokens, and, similarly, $\sqx_{t:t+k}$ for $(x_{t}, x_{t + 1}, \dots, x_{t + k})$. We use $[\sqx; \sqy]$ to denote the their concatenation. $[T]$ denotes integers in $[1, T]$. In the context of alignment, we will use $\sqx$ to denote a harmful prompt and $\sqy$ to be its response, where $\sqy = \sqr$ if it is a refusal and $\sqy = \sqh$ if it is a harmful response. The safety of models under an attack will be measured by attack success rate~(\textbf{ASR}) on a given safety evaluation dataset, which is the ratio of harmful outputs generated by the model under the given attack. 
\section{Understanding Safety Alignment via Autoregressive Consistency}
\label{sec:theory}
The shallow safety alignment phenomenon, in which the model's safe behavior is strongest near the first few tokens, can cause aligned LLMs to comply with harmful user queries once generation departs from the regular refusal prefixes of safe responses. This observation raises a fundamental question: why should safety alignment naturally concentrate near the beginning of the response, even when the model is supervised on an entire refusal trajectory?

In this section, we analyze safety alignment through the lens of autoregressive consistency.
After formalizing the concept in Definition~\ref{def:consistency}, we study the \textbf{learning dynamics} of safety alignment and show that autoregressive consistency can make later refusal tokens less informative for learning, leading alignment updates to concentrate near the beginning of the response while leaving late positions largely unchanged~(Section~\ref{sec:ac_to_shallow}). We then show that the same mechanism suggests \textbf{a broader class of attacks}: attacks need not target only the start of generation, but can also bypass safety alignment by inducing harmful continuation states inside the output trajectory, with random insertion attack as one concrete example (Section~\ref{sec:broader_attack}). Thus, whereas Qi et al.~\citep{qi2025safety} identified shallow safety alignment and motivated deeper safety alignment, we explain why such shallowness can arise and argue that the underlying failure mechanism extends beyond the shallow issue.
\subsection{Shallow Safety Alignment Can Arise from Autoregressive Consistency}
\label{sec:ac_to_shallow}
In certain natural language domains, a sufficiently long prefix is expected to induce a near-deterministic conditional distribution over continuations~\citep{piantadosi2014zipf,raychev2014code,shannon1951english}. This arises in text categories that are low-entropy by design---where the space of expected responses is severely restricted by mere syntactical and semantic consistencies, such that $p(x_{t + 1 } \mid \sqx_{:t})$ is close to 1 for $t$ larger than some critical value $t_c$. A canonical example is programming codes. As the output trajectory of an autoregressive model $p_{\theta_{\base}}$  is trained well enough to capture such tendencies, one may expect that  $p_{\theta_{\base}}$ to satisfy the same property. We formalize this property as the model's autoregressive consistency with the data:
\begin{definition}[Autoregressive consistency]
\label{def:consistency}
The base autoregressive model $p_{\theta_{\base}}$ has autoregressive consistency on $\mathcal{D}_{\rm{data}}$, if $\ \forall \sqx \in \mathcal{D}_{\rm{data}}$ and $\ \forall \epsilon>0,$
    \begin{equation}
     \exists t_c > 0,  \ \text{such that }\forall t > t_c: \ \begin{aligned}
         p_{\mathrm{\theta_{base}}}(x_{t+ 1} \mid \sqx_{:t}) > 1-\epsilon.
    \end{aligned}
\end{equation}
\end{definition}
Intuitively, the output trajectory of autoregressive models is obtained by repeatedly conditioning on their own generated tokens. As a result, next-token prediction tends to sustain and extend the current prefix in a locally consistent way. The work of Qi et al.~\cite{qi2025safety} strongly suggests that the data commonly used in safety alignment training enable the model to have such autoregressive consistency in the first place, e.g., forcing the base~(unaligned) model to start with refusal prefix can make them maintain that refusal response, that is
\begin{figure}
    \centering
    \begin{subfigure}[b]{0.49\textwidth}
        \centering
        \includegraphics[width=\textwidth]{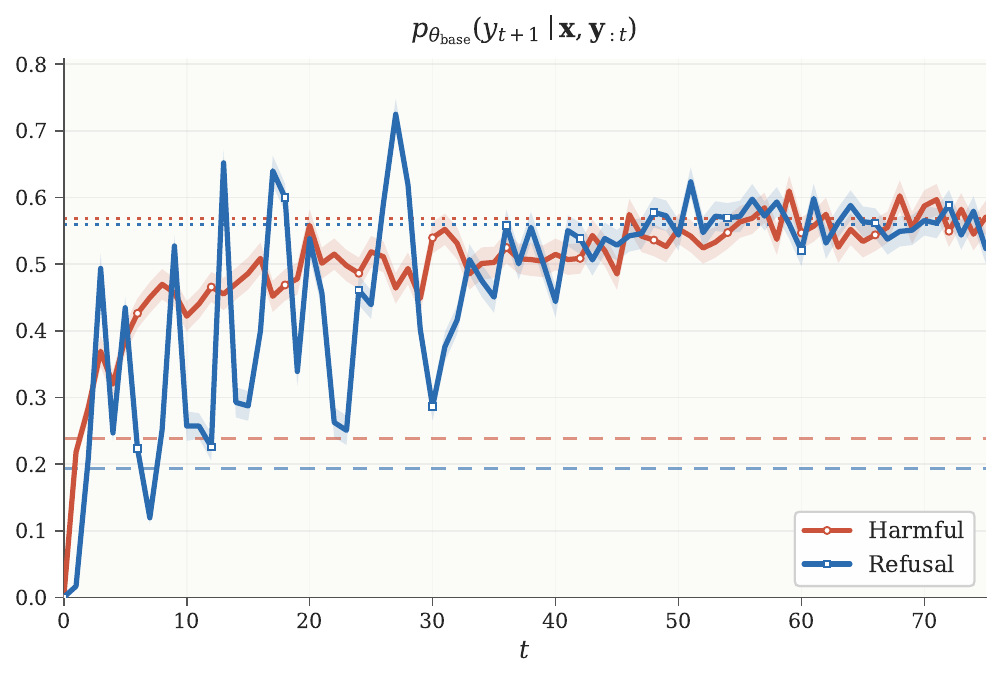}
        \caption{\texttt{Llama-2-7B}.}
        \label{subfig:1} 
    \end{subfigure}
    \hfill  
    \begin{subfigure}[b]{0.49\textwidth}
        \centering
        \includegraphics[width=\textwidth]{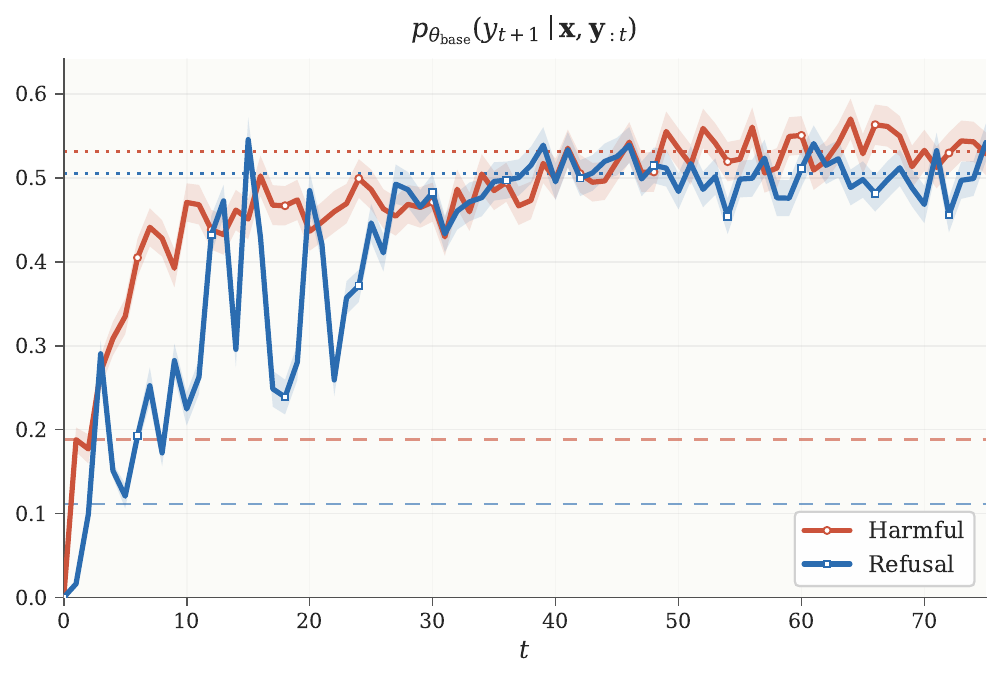}
        \caption{\texttt{Llama-3.1-8B}.}
        \label{subfig:2}
    \end{subfigure}
    \caption{\small \textbf{Empirical measurement of next-token probability.}  For each continuation position $t$, we plot the mean probability assigned by the base model to the next token in the given trajectory, $p_{\theta_{\mathrm{base}}}(y_{t+1}\mid \sqx,\sqy_{:t})$, for harmful and refusal continuations. Shaded regions denote standard errors across examples. For each continuation type, the dashed horizontal line of the same color shows the mean over the first 5 continuation positions, and the dotted horizontal line of the same color shows the mean over the last 5 continuation positions. The results are computed on the dataset open-sourced by Qi et al.~\citep{qi2025safety}, which contains 256 harmful instructions, each paired with a refusal response and a harmful response~(see Appendix~\ref{app:exp} for more details).}
    \label{fig:assmp_validation}
\end{figure}
\begin{assumption}[Autoregressive consistency in safety alignment]
\label{assmp:consistency_alignment}
For safety alignment data pairs $(\sqx, \sqy)\sim \mathcal{D}_{\rm{align}}$, where  $\sqx$ is a harmful prompt and $\sqy$ is a paired response (which can be a safe refusal or a harmful response), the base autoregressive model $p_{\theta_{\base}}$ has autoregressive consistency on 
$\mathcal{D}_{\rm{align}}$, if $\ \forall (\sqx, \sqy)\in \mathcal{D}_{\rm{align}}$ and  $\ \forall \epsilon>0,$
    \begin{equation}
     \exists t_c > 0,  \ \text{such that }\forall t > t_c: \ \begin{aligned}
         p_{\mathrm{\theta_{base}}}(y_{t+1} \mid \sqx,\, \sqy_{:t}) & > 1-\epsilon.
    \end{aligned}
    \label{eq:data_shallow}
\end{equation}
\end{assumption}
That is, the base model is confident at locking in a refusal or harmful response once the first few response tokens are given. Fig.~\ref{fig:assmp_validation} provides empirical support for the soft analogue of Assumption~\ref{assmp:consistency_alignment}. For both \texttt{Llama-2-7B} and \texttt{Llama-3.1-8B}, the probability assigned to the trajectory-consistent next token $p_{\mathrm{\theta_{base}}}(y_{t+1} \mid \sqx,\, \sqy_{:t})$ increases sharply after the initial positions and stabilizes at a substantially higher level for both harmful and refusal continuations, respectively.  This suggests that, once a continuation branch has been established, the base model increasingly favors locally consistent continuation.  Although Assumption~\ref{assmp:consistency_alignment} is stated in a near-deterministic form for analytical clarity, exact-token probabilities around $0.5$ are already high in open-ended natural language, where probability mass is distributed across many semantically equivalent continuations.  Thus, these results can be interpreted as evidence for a soft empirical analogue of autoregressive consistency, and not as a literal verification of the idealized assumption.

In benign generation, autoregressive consistency is essential: it allows a model to maintain the topic, style, and reasoning continuity in a narrow scope. However, it may also induce a shortcut in safety alignment: move only the first few tokens onto a refusal trajectory and leave the rest of the generation to autoregressive consistency. We theoretically investigate whether safety alignment exploits this by analyzing the learning dynamics of alignment, beginning with a formal definition of shallow alignment.
\begin{definition}[Shallow alignment]
\label{def:shallow_alignment}
    Let $p_{\theta_{\mathrm{base}}}$ be the base model and  $p_{\theta_{\mathrm{aligned}}}$ be the aligned model. Suppose the prompt-response data pairs $(\sqx, \sqy)$  in an alignment data set $\mathcal{D}_\mathrm{align}$ follow a distribution $\pdata$. An alignment is shallow if the aligned model $p_{\theta_{\ald}}$ satisfies:  $\forall \varepsilon_1,  \varepsilon_2 > 0$, $\exists t_2 \geq t_1 > 0$  such that 
    \begin{align}
        \mathrm{KL}\Big[\pdata(\cdot \, | \sqx,\sqy_{:t}  )\, |\, p_{\theta_{\ald}}(\cdot \,| \sqx, \sqy_{:t} )\Big] & < \varepsilon_1, \qquad    0\leq t <  t_1,\label{eq:shallow_condition1}\\
     \left|\left| p_{\theta_{\ald}}(\cdot \,| \sqx, \sqy_{:t}) - p_{\mathrm{\theta_{base}}}(\cdot \, | \sqx, \sqy_{:t})\right|\right| & < \varepsilon_2, \qquad  t > t_2,\label{eq:shallow_condition2}
    \end{align}
where $\mathrm{KL}$ denotes Kullback–Leibler divergence.
\end{definition}
The first condition means that the model is well aligned up to the first few response tokens and the second condition means that the alignment training did not move the model far from its base model for late token generations. In practice, this means shallow alignment makes sure that the first few tokens are refusal keywords for a harmful prompt, and the remaining refusal generation is guaranteed by autoregressive consistency. However, if a harmful prefix with length larger than $t_2$ is given, we expect a harmful response to continue. In the following, we take the first condition as a given and mainly investigate the second. 

We now reveal that the shallow alignment defined above can arise from autoregressive consistency. Consider alignment training initialized from the base autoregressive model $\theta_{\base}$ and optimization objectives of SFT or DPO: 
\begin{align}
    L_{\mathrm{SFT}}(\theta) & = - \sum_t\E_{(\sqx, \sqr)\sim \gD} \left[\log  p_\mathrm{\theta}(r_{t+1}|\sqx, \sqr_{:t})\right],\label{eq:sft_obj}\\
    L_{\mathrm{DPO}}(\theta) & = -\E_{(\sqx, \sqr, \sqh)\sim \gD}\left[
        \log \sigma\!
        \left( \beta 
            \left[ 
                \log \frac{p_\theta(\sqr\mid \sqx)}{p_{\theta_{\mathrm{base}}}(\sqr\mid \sqx)} - \log \frac{p_\theta(\sqh\mid \sqx)} {p_{\theta_{\mathrm{base}}}(\sqh\mid \sqx)}
            \right]
        \right)
    \right].\label{eq:dpo_obj}
\end{align}
Although these objectives differ in form, their gradients similarly decompose into token-level terms involving next-token log-probabilities and their parameter gradients $\nabla_\theta p_\theta(y_{t+1}\mid \sqx,\sqy_{:t})$. We therefore analyze the learning dynamics following the strategy below.

\paragraph{Analysis strategy.} Our analysis proceeds in two stages.
First, we separate out the part of the argument that does not depend on the particular alignment objective: under autoregressive consistency, the late-position next-token distribution $p_\theta(\cdot \mid \sqx,\sqy_{:t})$ has a small gradient with respect to the model parameters~(Proposition~\ref{prop:gradient_concentration}). Second, we apply this observation to the SFT and DPO objectives in Eq.~\eqref{eq:sft_obj} and Eq.~\eqref{eq:dpo_obj}, showing that their late-position gradient contributions are indeed suppressed~(Proposition~\ref{prop:loss-func-grad}). As a result, the effective learning signal concentrates only near the early tokens that initiate refusal, giving a learning dynamics explanation of shallow alignment under gradient descent. The final statement is summarized in Theorem~\ref{thm:ac_to_shallow}.

We start with the first stage, which relies on the bound on softmax Jacobian (proof in Appendix~\ref{app:proof-Jacobian-bound}).
\begin{lemma}[Bound on softmax Jacobian]
\label{lemma:Jacobian-bound}
Under Assumption \ref{assmp:consistency_alignment}, the Frobenius norm of the base model softmax Jacobian satisfies
\begin{equation}
    \|A_t\|_F^2 < 2\epsilon, \qquad\forall t>t_c,
\end{equation}
where $(A_t)_{ij} = \nabla_{z_i} p_\mathrm{\theta_{base}}(j\textup{-th token} \,|\sqx, \sqy_{:t} )$ and $z_i =z_\theta(i\textup{-th token}|\sqx, \sqy_{:t} )$ is the output logit for the $i$-th token in the vocabulary before final softmax layer. 
\end{lemma}
As a result, the conditional probability has a small magnitude at late positions, as shown below.
\begin{prop}[Gradient concentration of  conditional probability]
\label{prop:gradient_concentration}
    Under Assumption~\ref{assmp:consistency_alignment}, and assuming $\|\nabla_{\theta} z\|_F$ is upper bounded by a constant $C$, where $z$ is the logit vector defined in Lemma~\ref{lemma:Jacobian-bound}, then at initialization 
    \begin{equation}
    \big[\|\nabla_{\theta}p_\theta(\cdot \,|\sqx, \sqy_{:t})\|_F \big]_{\theta =\theta_{\rm{base}}} < C(2\epsilon)^\frac{1}{2}, \qquad \forall t>t_c.
    \end{equation} 
\end{prop}
\begin{proof}
    By chain rule $\nabla_{\theta}p_\theta(\cdot \,|\sqx, \sqy_{:t})= \nabla_{z}p_\theta(\cdot\,|\sqx, \sqy_{:t}) \cdot \nabla_{\theta} z$. At initialization, by Lemma~\ref{lemma:Jacobian-bound} we have 
     $\|\nabla_{\theta}p_\theta(\cdot \,|\sqx, \sqy_{:t})\|_F = \|A_t\cdot \nabla_{\theta} z\|_F \leq \|A_t\|_F \|\nabla_{\theta} z\|_F < C(2\epsilon)^\frac{1}{2}$.
\end{proof}
Proposition~\ref{prop:gradient_concentration} is independent of the particular form of the alignment objective: it shows that next-token distributions at late positions have small parameter gradients under autoregressive consistency. We can specialize this bound to the SFT and DPO objectives in Eq.~\eqref{eq:sft_obj} and Eq.~\eqref{eq:dpo_obj}.  This yields a corresponding concentration property for their loss gradients, summarized in the following proposition~(proof in Appendix~\ref{app:proof-loss-func-grad}).
\begin{prop}[Concentration of alignment gradients]
\label{prop:loss-func-grad}
For SFT (Eq.~\eqref{eq:sft_obj}) or DPO (Eq.~\eqref{eq:dpo_obj}) on a traditional alignment dataset (where autoregressive consistency holds), the dominant contributions to the loss gradient $\nabla_{\theta} L|_{\theta = \theta_\mathrm{base}}$ are from $\nabla_{\theta}p_\theta(\cdot \,|\sqx, \sqy_{:t})$ with $t\leq t_c$. 
\end{prop}
Proposition~\ref{prop:loss-func-grad} implies that $\nabla_{\theta} L|_{\theta = \theta_\mathrm{base}}$ effectively truncates at the $t_c$-th term, i.e., autoregressive consistency suppresses gradient contributions at late positions for the base model. As a result, a gradient update from the base model primarily changes next-token probabilities near early positions $t < t_c$, while only weakly affecting those at late positions (Proposition~\ref{thm:ac_to_shallow_one_step}), and thus a similar autoregressive consistency holds after the first gradient update (Corollary~\ref{coro:small-update}). The same argument extends to multi-step gradient descent on the base model by induction: once the first update leaves next-token probabilities at late positions close to the base model, autoregressive consistency continues to hold there in a weakened form, so subsequent updates remain similarly suppressed; iterating this argument over multiple gradient steps yields the shallow alignment phenomenon: alignment training mainly modifies the early tokens, while the next-token distributions at later positions remain close to those of the base model. The following theorem summarizes this effect~(proof in Appendix.~\ref{app:multistep_shallow_update}).
\begin{theorem}[Shallow alignment from autoregressive consistency]
\label{thm:ac_to_shallow}
Suppose the base model satisfies Assumption~\ref{assmp:consistency_alignment} with $\epsilon_0 = \epsilon$, and let $\theta_0 = \theta_{\mathrm{base}}$ at the initialization. Consider gradient descent for a fixed iteration count $K \geq 1$
\[
  \theta_{k+1}=\theta_k-\eta\nabla L(\theta_k), \quad k =0, 1, \ldots, K - 1,  
\]
where $L$ is an alignment objective of SFT in Eq.~\eqref{eq:sft_obj} or DPO in Eq.~\eqref{eq:dpo_obj}. Assume that, along the optimization trajectory, the logit Jacobian and the second derivative  of the conditional next-token distribution $\|\nabla_\theta^2 p_{\theta}(\cdot \mid \sqx, \sqy_{: t})\|_F$  at late positions $ t > t_c$ are uniformly bounded by constants $C$ and $B$, respectively, and that the loss gradient is also upper bounded $\|\nabla L(\theta_k)\|\le G$. For $\varepsilon_2$ defined in Definition~\ref{def:shallow_alignment}, let $\alpha=\max \{C / \sqrt{2},\sqrt{B / 2} \}$, and let $\eta \le \frac{ \sqrt{\epsilon_0+\varepsilon_2} - \sqrt{\epsilon_0} }{ \alpha K G }$. Define $\epsilon_k$ recursively by
\begin{equation}
    \quad \epsilon_{k+1} = \epsilon_k + C\sqrt{2\epsilon_k}\eta\|\nabla L(\theta_k)\| + \frac{B}{2}\eta^2\|\nabla L(\theta_k)\|^2.    
\end{equation}

Then, the model satisfies a weaker form of autoregressive consistency after each gradient update
\begin{equation}
    \forall k\in[K],\ t > t_c: \quad p_{\theta_k}(y_{t+1}\mid \sqx, \sqy_{:t}) > 1 - \epsilon_k,        
\end{equation}
and SFT/DPO gradient contributions at late positions are also suppressed to $O(\sqrt{\epsilon_k})$.  Therefore, the effective alignment gradient remains concentrated on the early positions $t \le t_c$ while leaving late positions largely unchanged throughout training. As a result, the alignment is shallow, i.e., SFT/DPO changes the model mainly near the early response tokens while leaving next-token distributions  at late positions close to those of the base model 
\begin{equation}
   \|\Delta_{\theta_{\base}\to \theta_{\ald}} p_{\theta}(\cdot \mid \sqx, \sqy_{:t})\| := \|p_{\theta_K}(\cdot\mid \sqx, \sqy_{:t}) - p_{\theta_{\mathrm{base}}}(\cdot\mid \sqx, \sqy_{:t})\|_F \le \varepsilon_2, \quad t>t_c.
\end{equation}
\end{theorem}
\paragraph{Consequence: aligned model differs from base mainly at depth $t \leq t_c$.}This is precisely the  finding of Qi et al.~\citep{qi2025safety} the aligned model differs from the base model only at the first few output positions. 
The base model's existing capacity to continue a sufficiently long given prefix  is left entirely untouched by alignment training. When the said prefix is harmful, jailbreaking happens. 

Proposition~\ref{prop:loss-func-grad} and Theorem~\ref{thm:ac_to_shallow} together suggest that we might be able to mitigate shallow alignment by breaking harmful autoregressive consistency at arbitrarily late token positions, thereby enhancing the gradient norm of the loss functions. This will be the main topic of Section~\ref{sec:framework}.

\subsection{Attacks Need Not Target Only the Start}
\label{sec:broader_attack}
Our analysis of the learning dynamics of safety alignment identifies autoregressive consistency as a deeper failure mechanism. The same mechanism also exposes a broader class of  attacks targeting the trajectory. Attacks targeting prefixes are effective not only because they bypass shallow safety alignment near the beginning of generation, but also because they induce a harmful continuation state that autoregressive consistency can then preserve and extend. From this perspective, the start of generation is only one possible attack location. More generally, any attack that induces an autoregressive state from which the model favors harmful continuation over refusal can be dangerous, regardless of where that state appears in the output trajectory. This is the attack-side counterpart of Fig.~\ref{fig:summary}: harmful autoregressive consistency can be triggered throughout the trajectory, not only at the prefix.

Specifically, let $\sqx$ be a harmful prompt, let $\sqy$ be the output trajectory, let $\tau_{t} = (\sqx, \sqy_{:t})$ be the partial autoregressive state. Let $\Phi_{\theta}(\tau_t)$  measure how strongly the model favors harmful continuation over refusal response at $\tau_t$, which we call \emph{harmful autoregressive consistency margin}.
A simple example is $\Phi_{\theta}(\tau_t) = \log p_{\theta}(\sqh\mid \tau_t) - \log p_{\theta}(\sqr\mid \tau_t)$, where $\sqh$ is a harmful span and $\sqr$ is a refusal span. Then, for example, prefill attacks can be seen as constructing $\tilde{\tau}_t = (\sqx, \tilde{\sqy}_{:t})$, where $\tilde{\sqy}_{:t}$ is an unsafe assistant prefix and $t$ is small, such that $\Phi_\theta(\tilde{\tau}_t)$ is large. Similarly, input-side suffix attacks such as GCG~\citep{zou2023universaltransferableadversarialattacks} can be viewed as modifying the prompt context so that the model's initial autoregressive state has a large harmful margin. By exploiting autoregressive consistency throughout the output trajectory, the attack need not be confined to the start. We introduce random insertion attack as a concrete example below.
\paragraph{Random insertion attack.}If the model's autoregressive consistency is strong, the attacker may only need to insert a short harmful span to trigger harmful continuation. Random insertion attack works by forcing the model onto a harmful branch at a random position in the response, after which autoregressive consistency sustains that branch. Specifically, we construct the random insertion attack using a triplet dataset$(\sqx, \sqr, \sqh)$, where $\sqr$ is the refusal trajectory and $\sqh$ is a harmful response which can be short.
For each triplet, we take a short span $\sqh_{:q}$ of length $q$, and randomly sample a position $i$ from $|\sqr|$. This yields the final attack $\tilde{\tau}_{i + q} = (\sqx, [\sqr_{:i}; \sqh_{:q}])$~(Fig.~\ref{fig:random_insertion_attack}). The model is then asked to continue from $\tilde{\tau}_{i + q}$ according to $p_{\theta}(\cdot \mid \tilde{\tau}_{i + q})$. If the inserted span induces harmful autoregressive consistency, the model may leave the refusal trajectory and continue the harmful branch. In this sense, random insertion attack covers prefill attack as a special case in which the insertion position is restricted to the beginning of the output trajectory.
\begin{wrapfigure}{r}{0.49\textwidth}
    \centering
    \vspace{0.2cm}
    \includegraphics[width=0.48\textwidth]{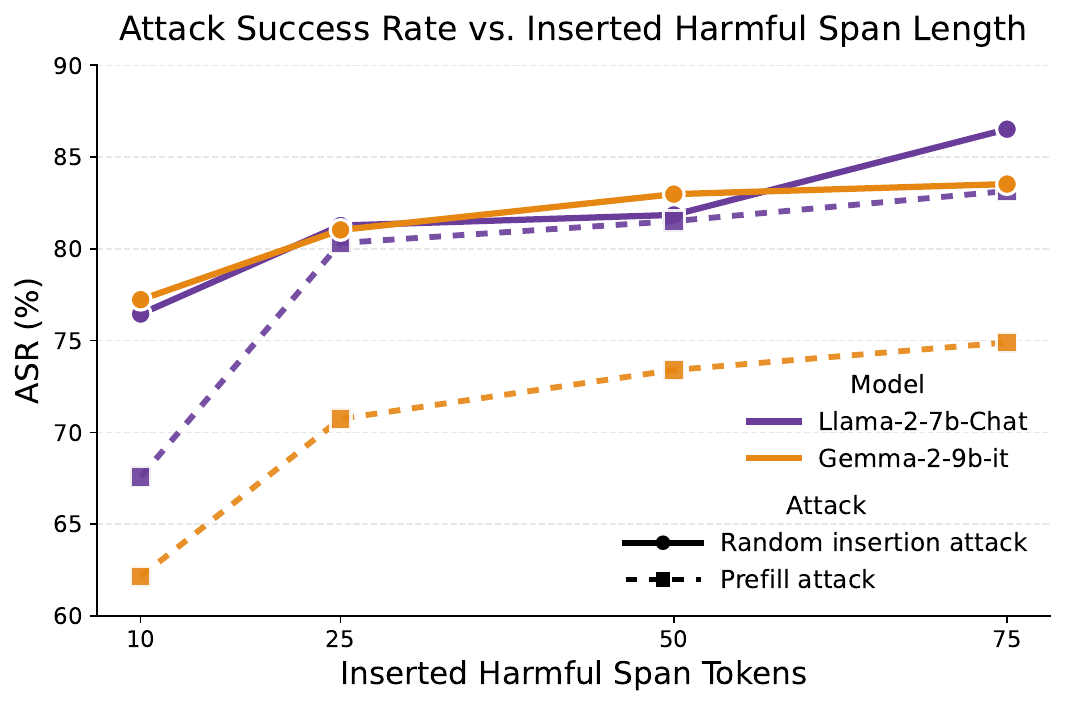}
    \caption{ASR vs. harmful span lengths for either prefill attack or random insertion attack evaluated on HEx-PHI safety benchmark~\citep{anonymous2024finetuning}.}
    \label{fig:random_attack_on_raw_models}
\end{wrapfigure}

The random insertion attack is intentionally simple. Its purpose is to demonstrate that harmful autoregressive consistency can be induced away from the start of generation. Even when the model is forced to begin---and to remain for many tokens---on a refusal trajectory, hence not ``shallow''~(Fig.~\ref{fig:random_insertion_attack}), a short harmful span inserted later can still redirect generation onto a harmful branch. Specifically, Fig.~\ref{fig:random_attack_on_raw_models} shows that, for both \texttt{Llama-2-7B-Chat}~\citep{touvron2023llama2openfoundation} and \texttt{Gemma-2-9b-it}~\citep{gemmateam2024gemma2improvingopen}, ASR for random insertion attacks are above 75\% and generally increases with harmful span length. In addition, even 10 tokens harmful span inserted in an otherwise safe refusal output can  dramatically redirect the generation to a harmful output, revealing the limitation of only protecting the prefix. Therefore, the vulnerability is not merely a failure of shallow prefixes, but also a consequence of the model’s tendency to preserve and extend the current autoregressive trajectory. Autoregressive consistency should therefore be treated as a crucial consideration in future safety alignment.

\section{Breaking Autoregressive Consistency by Adversarial Safety Alignment}
\label{sec:framework}
Qi et al.~\citep{qi2025safety} made an important step by identifying shallow safety alignment and arguing that safety alignment should be made deeper.  Our analysis in Section~\ref{sec:theory} suggests that depth alone is still an incomplete objective. Training deeper may extend the defended horizon against prefix-like attacks, but it does not by itself address the underlying dynamical problem.  As shown in Section~\ref{sec:broader_attack}, random insertion attack can remain effective even when the beginning of the output trajectory is already protected by a safe refusal (hence not ``shallow''). The broader failure mechanism arises from autoregressive consistency: once generation enters a harmful branch, whether near the beginning or in the middle of the output, the model tends to preserve and extend that branch. 
This implies that safety alignment should not only be trained to extend refusal farther from the beginning of generation. 

Then what should the objective be instead? Following the conceptual chain in Fig.~\ref{fig:summary}, we argue that safety alignment should also learn to \emph{break harmful autoregressive consistency} and recover from harmful continuation states throughout the output trajectory.  Following this principle, we introduce adversarial safety alignment in Section~\ref{sec:asa_framework} as an initial framework, and then present training with random worst-insertion attack in Section~\ref{sec:training_ra} as a practical first approximation to this framework. 

The purpose of adversarial safety alignment is to translate the preceding theoretical insight into a trainable objective. Our goal is not to present a final defense recipe, but to use a simple instantiation to study whether training against harmful continuation states can improve robustness beyond depth-specialized alignment.
\subsection{Adversarial Safety Alignment}
\label{sec:asa_framework}
Our core intuition is the following: given a refusal response $\sqr$, even from the worst corrupted continuation state along this trajectory~(i.e., a state from which the model can most easily establish and sustain a locally consistent harmful branch), training should make the model recover safe refusal behavior. In practice, we use the original refusal trajectory $\sqr$ as the recovery target. As in the conceptual chain in Fig.~\ref{fig:summary}, this encourages safety alignment to become robust to harmful continuation induced by autoregressive consistency throughout the output trajectory.

We therefore propose \textbf{adversarial safety alignment} as an initial framework inspired by adversarial training~\citep{madry2019deeplearningmodelsresistant}.
In standard adversarial training, the inner problem typically searches for a perturbation that maximizes the same loss minimized by the outer problem. Here, the roles of the two problems are different. The inner problem identifies a continuation state from which the current model most strongly favors harmful continuation, while the outer problem trains the model to recover safe refusal behavior from that state. Concretely, for each training example, we construct a set of candidate continuation states along the output trajectory, select the state that most strongly supports harmful autoregressive consistency, and then minimize a supervised refusal loss conditioned on the selected state.

Formally, given an LLM parameterized by $\theta$, \textbf{Adversarial Safety Alignment} is formulated as
\begin{equation}
\label{eq:asa_objective}
\begin{aligned}
      \min_{\theta}\; \E_{(\sqx, \sqr)\sim \gD} 
\Big[
    L_{\sft}\big( \sqr, \tau^{\star}_{\theta}(\sqx, \sqr); \theta\big)
\Big], \quad  s.t.\;\;  \tau^{\star}_{\theta}(\sqx, \sqr) \in \underset{\tilde{\tau}_t \in \gA_{\theta}(\sqx, \sqr)}{\arg\max} \ \Phi_{\theta}(\tilde{\tau}_t).
\end{aligned}
\end{equation}
Here, $(\sqx, \sqr)$ is a pair of harmful prompt $\sqx$ and a corresponding refusal response $\sqr$. The set $\gA_{\theta}(\sqx, \sqr)$ contains candidate harmful states adversarially constructed from $(\sqx, \sqr)$ along the refusal trajectory. The inner objective $\Phi_{\theta}(\tilde{\tau}_t)$, which we call harmful autoregressive consistency margin, measures how easily the current model can maintain harmful autoregressive consistency from the adversarially constructed state $\tilde{\tau}$. The outer loss
\[
    L_{\sft}\big( \sqr, \tau^{\star}_{\theta}(\sqx, \sqr); \theta\big) = -\log p_{\theta}\left( \sqr | \tau^{\star}_{\theta}(\sqx, \sqr)\right)
\]
then trains the model to recover the refusal response from the selected adversarial state.
Although this framework is motivated by the goal of breaking harmful autoregressive consistency, Eq.~\eqref{eq:asa_objective} should be understood as a practical surrogate rather than a direct optimization of a formal trajectory-level measure of that quantity.

A practical implementation of adversarial safety alignment therefore requires two design choices.
The first is the candidate set $\gA_{\theta}(\sqx,\sqr)$, which should cover harmful continuation states throughout the output trajectory. The second is the inner objective $\Phi_{\theta}(\tilde{\tau}_t)$, which should measure how strongly a perturbed state favors harmful continuation over refusal recovery.
We next present a first practical approximation to this framework by specifying concrete choices for both components.

\subsection{A Practical Instantiation: Adversarial Safety Alignment with Worst-Insertion Attack}
\label{sec:training_ra}
Since solving the adversarial safety alignment objective in Eq.~\eqref{eq:asa_objective} exactly is intractable, we instantiate it with a practical approximation based on random worst-insertion attack. This choice is directly motivated by Section~\ref{sec:broader_attack}: random insertion attack is a concrete example of the broader class of attacks that exploit harmful autoregressive consistency by inducing harmful continuation states along the output trajectory. To instantiate the framework, we specify both the candidate set $\gA$ and the inner objective $\Phi$.

\paragraph{Construction of $\gA$.}We construct $\gA$ using random insertion attack (Section~\ref{sec:broader_attack}). Because this construction uses a harmful span $\sqh$ in addition to the harmful prompt $\sqx$ and refusal trajectory $\sqr$, we denote the resulting candidate set by $\gA(\sqx, \sqr, \sqh)$. We construct it using a triplet dataset $(\sqx, \sqr, \sqh)$, where $\sqr$ is the refusal trajectory and $\sqh$ is a harmful response from which the inserted span is taken. For each triplet, we take a short harmful span $\sqh_{: q}$ of length $q$, and sample $k$ positions $\{i_j\}_{j=1}^k$ uniformly from $[|\sqr|]$. This yields the candidate set $\gA(\sqx,\sqr,\sqh)=\left\{(\sqx, [\sqr_{: i_j};\sqh_{ :q}])\right\}_{j=1}^k$. Each element of $\gA(\sqx,\sqr, \sqh)$ is thus a perturbed continuation state obtained by truncating the refusal trajectory at position $i_j$ and appending a short harmful span $\sqh_{:q}$, thereby forcing the model into a locally harmful state at that point. In this way, the candidate set covers diverse corrupted states throughout the refusal trajectory while keeping the inner search tractable.

\textbf{Inner objective $\Phi$.} For the inner objective, we compare harmful continuation with refusal recovery from the same perturbed state. We first prepare a bank of $m$ short candidate refusal continuations, $\gR=\{\sqr_j^{\cand}\}_{j=1}^m$, which we use to measure how strongly the model favors refusal recovery. Then, for each $\sqtau_{t}\in\gA(\sqx,\sqr, \sqh)$, we define the harmful autoregressive consistency margin as 
\[
    \Phi_{\theta}(\tau_t)
    =
    H(\sqtau_t;\theta)-R(\sqtau_t;\theta).
\]
Here, 
\[
    H(\sqtau_t;\theta) := \frac{1}{T_o}\log p_\theta\!\left(\sqh_{q + 1:q+T_o}\mid \sqtau_t\right)
\]
measures how strongly the model continues the harmful branch from the perturbed state. We interpret this as a practical score of harmful autoregressive consistency. 
\[
    R(\sqtau_t;\theta):=\max_{\sqr^{\cand}\in\gR} \frac{1}{|\sqr^{\cand}|}\log p_\theta\!\left(\sqr^{\cand}\mid \sqtau_t\right)
\]
measures how easily the model can restart refusal from that same state, where the $\max$ operator selects the single refusal candidate $\sqr^{\cand}$ that the model most prefers. Thus, maximizing $\Phi_\theta$ selects the inserted candidate state where the model most strongly favors harmful continuation over refusal recovery. We refer to this approximation as random worst-insertion attack.

\subsection{Experiments}
\label{sec:exp}
\paragraph{Implementation details.}For adversarial safety alignment with random worst-insertion attack, we use the objective Eq.~\eqref{eq:asa_objective}, denoted by $L_{\mathrm{Insertion}}$ for simplicity, for a triplet dataset $(\sqx, \sqr, \sqh)\sim D_{\safety}$ with $\gA$ and $\Phi$ specified above and add a utility loss $L_{\uti}$ on a utility dataset $(\sqz, \sqy)\sim D_{\uti}$ to maintain model utility. Thus, the overall objective is
\begin{equation}
    \min_{\theta} \ \lambda L_{\mathrm{Insertion}} + (1 - \lambda) L_{\uti}, \quad \text{where }\  L_{\uti}:= - \E_{(\sqz, \sqy)\sim D_{\uti}}\Big[\log p_{\theta}(\sqy | \sqz) \Big].
\end{equation}
We refer to the method as \texttt{Insertion} for simplicity below. For comparison, we study two baselines, \texttt{Clean} and \texttt{Prefill}. Specifically, (1) For \texttt{Clean}, we fine-tune the model directly on the clean harmful prompt and refusal response pair $(\sqx, \sqr, \sqh)\sim D_{\safety}$ by standard SFT~(we do not use the harmful response $\sqh$, hence the name \texttt{Clean}), i.e., the objective is 
\[
    \min_{\theta}\ - \lambda \E_{(\sqx, \sqr, \sqh)\sim D_{\safety}}\Big[\log p_{\theta}(\sqr | \sqx) \Big] + (1 - \lambda)L_{\uti};
\]
(2) For \texttt{Prefill}, we apply the deep alignment method from Qi et al.~\citep{qi2025safety} by fine-tuning with SFT on an augmented dataset $([\sqx;\, \sqh_{:q}], \sqr)$ for $(\sqx, \sqr, \sqh) \sim D_{\safety}$, i.e., the objective is
\[
    \min_{\theta}\ - \lambda \E_{(\sqx, \sqr, \sqh)\sim D_{\safety}}\Big[\log p_{\theta}(\sqr | \sqx, \sqh_{:q}) \Big] + (1 - \lambda)L_{\uti},
\]
which corresponds to making safety alignment deeper against prefix-style attacks. This objective can also be seen as adversarial safety alignment with the inner problem solved by prefill attack, hence the simplified name \texttt{Prefill}.

In our experiment, we set $\lambda=0.5$ for all three methods, and use $q=50$ for both \texttt{Prefill} and \texttt{Insertion}\footnote{A natural variant is to sample $q$ from a distribution, such as a uniform distribution, for each data point. Since our primary goal here is to study the underlying mechanism, we focus on the simplest setting and leave this generalization to future work.}. For \texttt{Insertion}, we set the number of candidate positions in $\gA$ to $k=32$. For the inner objective $\Phi$, we use $m=4$ for the bank of candidate refusal restart. We use \texttt{Llama-2-7b-Chat} model and the safety data from Qi et al.~\citep{qi2025safety} as $D_{\safety}$, respectively. The utility dataset $D_{\uti}$ is from Alpaca data~\citep{alpaca}. We evaluate ASR on the HEx-PHI safety benchmark and defer further details to Appendix~\ref{app:exp}.
\begin{table}
\centering
\caption{\textbf{ASR (\%) $\downarrow$} comparison across attacks with harmful span length 25, 50, 75, and 100. }
\vspace{0.1cm}
\begin{tabular}{c|cccc|cccc}
\toprule
\multirow{2}{*}{\small{\textbf{Method}}}
& \multicolumn{4}{c|}{\small{\textbf{Random Insertion Attack}}}
& \multicolumn{4}{c}{\small{\textbf{Prefill Attack}}}\\
\cline{2-9}
&  \small{25} & \small{50} & \small{75} & \small{100}
&  \small{25} & \small{50} & \small{75} & \small{100}\\
\hline
\small{\texttt{Clean}}
& \tiny{$76.6\pm 2.2$} & \tiny{$80.4\pm 1.4$} & \tiny{$83.6\pm 3.1$} & \tiny{$85.8\pm 3.0$} 
& \tiny{$75.3\pm 3.0$} & \tiny{$80.4\pm 2.0$} & \tiny{$81.6\pm 2.7$} & \tiny{$84.3\pm 2.5$}\\
\small{\texttt{Prefill}}
& \tiny{$18.4\pm 1.0$} & \tiny{$22.3\pm 1.3$} & \tiny{$27.6\pm 0.9$} & \tiny{$34.3\pm 0.4$}
& \tiny{$3.8\pm 1.1$} & \tiny{$9.0\pm 1.8$} & \tiny{$13.2\pm 1.6$} & \tiny{$20.9\pm 2.2$}\\
\small{\texttt{Insertion}}
& \tiny{$2.7\pm 0.3$} & \tiny{$3.4\pm 1.0$} & \tiny{$4.8\pm 0.2$} & \tiny{$5.9\pm 0.4$}
& \tiny{$6.5\pm 1.8$} & \tiny{$11.0\pm 2.7$} & \tiny{$13.5\pm 2.4$} & \tiny{$21.2\pm 2.3$}\\
\bottomrule
\end{tabular}
\label{tab:asr_comparison}
\end{table}
\paragraph{Results.}Tab.~\ref{tab:asr_comparison} reports ASR for random insertion attack~(Section~\ref{sec:broader_attack}) and prefill attack across harmful span lengths 25, 50, 75, and 100 tokens  on \texttt{Llama-2-7b-Chat} models (further) aligned with \texttt{Clean}, \texttt{Prefill}, and \texttt{Insertion}, respectively. Results are reported as mean $\pm$ standard deviation across three runs. Tab.~\ref{tab:asr_comparison} shows three main findings: (1) \texttt{Clean} remains highly vulnerable to both attacks, even though both attacks are simple and inexpensive, indicating that standard safety alignment remains brittle; (2) \texttt{Prefill} sharply reduces ASR on the matched prefill attack, but transfers poorly to random insertion attack, confirming that making alignment deeper helps but still is incomplete; (3) \texttt{Insertion} achieves the lowest ASR on random insertion attack while remaining competitive on prefill attack, suggesting that adversarial safety alignment can better mitigates the underlying failure mechanism induced by autoregressive consistency. 

\paragraph{Evaluation on other attacks.}\begin{wraptable}{r}{0.51\textwidth}\vspace{-0.3cm}
\centering
\caption{\textbf{ASR (\%) $\downarrow$} of different attacks on \texttt{Llama-2-7b-Chat}.}
\begin{tabular}{c|cccc}
    \toprule
     & \small{\textbf{GCG}} & \small{\textbf{PAIR}} & \small{\textbf{TAP}} & \small{\textbf{AutoDAN}}\\
    \hline
    \small{\texttt{Initial}}
    & \small{34.5} & \small{7.5} & \small{5.5} & \small{0.5}
     \\
    \small{\texttt{Insertion}}
    & \small{2.5} & \small{1.0} & \small{1.0} & \small{0.0}
     \\
    \bottomrule
\end{tabular}
\label{tab:other_attacks}
\end{wraptable} We further evaluate on several common attacks, including: (1) GCG, a white-box suffix attack; (2) PAIR~\citep{chao2024jailbreakingblackboxlarge}, a black-box attack that uses an attacker LLM to generate prompts; (3) TAP~\citep{mehrotra2024treeattacksjailbreakingblackbox}, a black-box attack that utilizes an attacker LLM to iteratively refine prompts; and (4) AutoDAN~\citep{liu2024autodan} which performs adversarial input perturbations. We use the raw \texttt{Llama-2-7b-Chat} model (named as \texttt{Initial}) and a further aligned version obtained by adversarial safety alignment (named as \texttt{Insertion} as before). We use the HarmBench~\citep{mazeika2024harmbenchstandardizedevaluationframework} for the evaluation, where we take the 200 standard behaviors from it as our evaluation dataset due to the significant computational cost of attacks like GCG. The results are summarized in Tab.~\ref{tab:other_attacks}, showing that adversarial safety alignment remains competitive beyond random insertion attack.

\paragraph{Evaluation on other models.}In Tab.~\ref{tab:asr_comparison_more_models}, we repeat the experiments from Tab.~\ref{tab:asr_comparison} on two additional models: \texttt{Llama-3.1-8B-Instruct}~\citep{grattafiori2024llama3herdmodels} and \texttt{Gemma-2-9b-it}~\citep{gemmateam2024gemma2improvingopen} to validate that our conclusions also hold for these models. The results show that the main finding in Tab.~\ref{tab:asr_comparison} transfers across model families. On both \texttt{Llama-3.1-8B-Instruct} and \texttt{Gemma-2-9b-it}, adversarial safety alignment substantially improves robustness under random insertion attack, while still remaining competitive on prefill attack.
\begin{table}[h!]
\centering
\vspace{-0.2cm}
\caption{Comparison of \textbf{ASR (\%) $\downarrow$} on different models across attacks with harmful span length 25, 50, 75, and 100. These models are trained using the same set of parameters as in Tab.~\ref{tab:asr_comparison}}
\begin{tabular}{cc|cccc|cccc}
    \toprule
    \multirow{2}{*}{\small \textbf{Model}} & \multirow{2}{*}{\small  \textbf{Method}}
    & \multicolumn{4}{c|}{\small \textbf{Random Insertion Attack}}
    & \multicolumn{4}{c}{\small \textbf{Prefill Attack}} \\
    \cline{3-10}
    & & \small 25 & \small 50 & \small 75 & \small 100
      & \small 25 & \small 50 & \small 75 & \small 100 \\
    \hline
    \multirow{2}{*}{\small{\texttt{Llama-3.1-8B-Instruct}}}
    & {\small \texttt{Prefill}} & \small{18.3} & \small{23.7} & \small{27.8} & \small{35.3} & \small{1.3} & \small{4.2} & \small{20.6} & \small{29.0} \\
    & {\small \texttt{Insertion}} &  \small{1.9} & \small{2.5} & \small{6.4} & \small{10.4} & \small{1.6} & \small{4.2} & \small{10.9} & \small{23.1} \\
    \hline
    \multirow{2}{*}{{\small \texttt{Gemma-2-9b-it}}}
    & {\small \texttt{Prefill}}      & \small{13.2} & \small{17.0} & \small{23.3} & \small{24.8} & \small{7.1} & \small{13.5} & \small{22.0} & \small{22.6} \\
    & {\small \texttt{Insertion}}  & \small{0.3} & \small{2.1} & \small{6.3} & \small{6.6} & \small{13.8} & \small{18.1} & \small{27.1} & \small{29.2}  \\
    \bottomrule
\end{tabular}
\label{tab:asr_comparison_more_models}
\vspace{-0.1cm}
\end{table}
\vspace{-0.1cm}
\section{Discussion}
\label{sec:discuss}
\vspace{-0.1cm}
Our work explains shallow safety alignment through the lens of autoregressive consistency by analyzing the learning dynamics of safety alignment. We show that autoregressive consistency can concentrate alignment updates near the early response tokens, providing a mechanistic explanation for why safety alignment can become shallow. We further argue that the same mechanism can hurt safety alignment by preserving and extending a harmful continuation once it has begun. This suggests a broader class of attacks that induce harmful continuation states inside the output trajectory, with random insertion attack as one concrete example.

These observations suggest that making alignment deeper is helpful but incomplete, because it does not directly target the underlying harmful mechanism. Safety alignment should also train models to break harmful autoregressive consistency and recover safe behavior from harmful continuation states. As an initial step, we propose an adversarial safety alignment framework and instantiate it with random worst-insertion training. We stress that the goal of this work is not to present a final defense recipe, but to identify harmful autoregressive consistency as a mechanism underlying safety fragility and to study the alignment objectives suggested by this mechanism. We hope this work encourages future alignment and attack methods to explicitly consider the role of autoregressive consistency throughout generation.
\bibliography{main}

@inproceedings{
wei2022finetuned,
title={Finetuned Language Models are Zero-Shot Learners},
author={Jason Wei and Maarten Bosma and Vincent Zhao and Kelvin Guu and Adams Wei Yu and Brian Lester and Nan Du and Andrew M. Dai and Quoc V Le},
booktitle={International Conference on Learning Representations},
year={2022},
url={https://openreview.net/forum?id=gEZrGCozdqR}
}

@inproceedings{
rafailov2023direct,
title={Direct Preference Optimization: Your Language Model is Secretly a Reward Model},
author={Rafael Rafailov and Archit Sharma and Eric Mitchell and Christopher D Manning and Stefano Ermon and Chelsea Finn},
booktitle={Thirty-seventh Conference on Neural Information Processing Systems},
year={2023},
url={https://openreview.net/forum?id=HPuSIXJaa9}
}

@misc{bai2022traininghelpfulharmlessassistant,
      title={Training a Helpful and Harmless Assistant with Reinforcement Learning from Human Feedback}, 
      author={Yuntao Bai and Andy Jones and Kamal Ndousse and Amanda Askell and Anna Chen and Nova DasSarma and Dawn Drain and Stanislav Fort and Deep Ganguli and Tom Henighan and Nicholas Joseph and Saurav Kadavath and Jackson Kernion and Tom Conerly and Sheer El-Showk and Nelson Elhage and Zac Hatfield-Dodds and Danny Hernandez and Tristan Hume and Scott Johnston and Shauna Kravec and Liane Lovitt and Neel Nanda and Catherine Olsson and Dario Amodei and Tom Brown and Jack Clark and Sam McCandlish and Chris Olah and Ben Mann and Jared Kaplan},
      year={2022},
      eprint={2204.05862},
      archivePrefix={arXiv},
      primaryClass={cs.CL},
      url={https://arxiv.org/abs/2204.05862}, 
}

@inproceedings{
qi2025safety,
title={Safety Alignment Should be Made More Than Just a Few Tokens Deep},
author={Xiangyu Qi and Ashwinee Panda and Kaifeng Lyu and Xiao Ma and Subhrajit Roy and Ahmad Beirami and Prateek Mittal and Peter Henderson},
booktitle={The Thirteenth International Conference on Learning Representations},
year={2025},
url={https://openreview.net/forum?id=6Mxhg9PtDE}
}

@misc{ouyang2022traininglanguagemodelsfollow,
      title={Training language models to follow instructions with human feedback}, 
      author={Long Ouyang and Jeff Wu and Xu Jiang and Diogo Almeida and Carroll L. Wainwright and Pamela Mishkin and Chong Zhang and Sandhini Agarwal and Katarina Slama and Alex Ray and John Schulman and Jacob Hilton and Fraser Kelton and Luke Miller and Maddie Simens and Amanda Askell and Peter Welinder and Paul Christiano and Jan Leike and Ryan Lowe},
      year={2022},
      eprint={2203.02155},
      archivePrefix={arXiv},
      primaryClass={cs.CL},
      url={https://arxiv.org/abs/2203.02155}, 
}

@misc{zou2025representationengineeringtopdownapproach,
      title={Representation Engineering: A Top-Down Approach to AI Transparency}, 
      author={Andy Zou and Long Phan and Sarah Chen and James Campbell and Phillip Guo and Richard Ren and Alexander Pan and Xuwang Yin and Mantas Mazeika and Ann-Kathrin Dombrowski and Shashwat Goel and Nathaniel Li and Michael J. Byun and Zifan Wang and Alex Mallen and Steven Basart and Sanmi Koyejo and Dawn Song and Matt Fredrikson and J. Zico Kolter and Dan Hendrycks},
      year={2025},
      eprint={2310.01405},
      archivePrefix={arXiv},
      primaryClass={cs.LG},
      url={https://arxiv.org/abs/2310.01405}, 
}

@misc{sheshadri2025latentadversarialtrainingimproves,
      title={Latent Adversarial Training Improves Robustness to Persistent Harmful Behaviors in LLMs}, 
      author={Abhay Sheshadri and Aidan Ewart and Phillip Guo and Aengus Lynch and Cindy Wu and Vivek Hebbar and Henry Sleight and Asa Cooper Stickland and Ethan Perez and Dylan Hadfield-Menell and Stephen Casper},
      year={2025},
      eprint={2407.15549},
      archivePrefix={arXiv},
      primaryClass={cs.LG},
      url={https://arxiv.org/abs/2407.15549}, 
}

@misc{xhonneux2024efficientadversarialtrainingllms,
      title={Efficient Adversarial Training in LLMs with Continuous Attacks}, 
      author={Sophie Xhonneux and Alessandro Sordoni and Stephan Günnemann and Gauthier Gidel and Leo Schwinn},
      year={2024},
      eprint={2405.15589},
      archivePrefix={arXiv},
      primaryClass={cs.LG},
      url={https://arxiv.org/abs/2405.15589}, 
}

@misc{casper2025defendingunforeseenfailuremodes,
      title={Defending Against Unforeseen Failure Modes with Latent Adversarial Training}, 
      author={Stephen Casper and Lennart Schulze and Oam Patel and Dylan Hadfield-Menell},
      year={2025},
      eprint={2403.05030},
      archivePrefix={arXiv},
      primaryClass={cs.CR},
      url={https://arxiv.org/abs/2403.05030}, 
}

@misc{xu2024safedecodingdefendingjailbreakattacks,
      title={SafeDecoding: Defending against Jailbreak Attacks via Safety-Aware Decoding}, 
      author={Zhangchen Xu and Fengqing Jiang and Luyao Niu and Jinyuan Jia and Bill Yuchen Lin and Radha Poovendran},
      year={2024},
      eprint={2402.08983},
      archivePrefix={arXiv},
      primaryClass={cs.CR},
      url={https://arxiv.org/abs/2402.08983}, 
}

@inproceedings{
zhao2026llms,
title={{LLM}s Encode Harmfulness and Refusal Separately},
author={Jiachen Zhao and Jing Huang and Zhengxuan Wu and David Bau and Weiyan Shi},
booktitle={The Thirty-ninth Annual Conference on Neural Information Processing Systems},
year={2026},
url={https://openreview.net/forum?id=zLkpt30ngy}
}

@misc{ganguli2022redteaminglanguagemodels,
      title={Red Teaming Language Models to Reduce Harms: Methods, Scaling Behaviors, and Lessons Learned}, 
      author={Deep Ganguli and Liane Lovitt and Jackson Kernion and Amanda Askell and Yuntao Bai and Saurav Kadavath and Ben Mann and Ethan Perez and Nicholas Schiefer and Kamal Ndousse and Andy Jones and Sam Bowman and Anna Chen and Tom Conerly and Nova DasSarma and Dawn Drain and Nelson Elhage and Sheer El-Showk and Stanislav Fort and Zac Hatfield-Dodds and Tom Henighan and Danny Hernandez and Tristan Hume and Josh Jacobson and Scott Johnston and Shauna Kravec and Catherine Olsson and Sam Ringer and Eli Tran-Johnson and Dario Amodei and Tom Brown and Nicholas Joseph and Sam McCandlish and Chris Olah and Jared Kaplan and Jack Clark},
      year={2022},
      eprint={2209.07858},
      archivePrefix={arXiv},
      primaryClass={cs.CL},
      url={https://arxiv.org/abs/2209.07858}, 
}

@misc{madry2019deeplearningmodelsresistant,
      title={Towards Deep Learning Models Resistant to Adversarial Attacks}, 
      author={Aleksander Madry and Aleksandar Makelov and Ludwig Schmidt and Dimitris Tsipras and Adrian Vladu},
      year={2019},
      eprint={1706.06083},
      archivePrefix={arXiv},
      primaryClass={stat.ML},
      url={https://arxiv.org/abs/1706.06083}, 
}

@inproceedings{
anonymous2024finetuning,
title={Fine-tuning Aligned Language Models Compromises Safety, Even When Users Do Not Intend To!},
author={Xiangyu Qi and Yi Zeng and Tinghao Xie and Pin-Yu Chen and Ruoxi Jia and Prateek Mittal and Peter Henderson},
booktitle={The Twelfth International Conference on Learning Representations},
year={2024},
url={https://openreview.net/forum?id=hTEGyKf0dZ}
}

@misc{alpaca,
  author = {Rohan Taori and Ishaan Gulrajani and Tianyi Zhang and Yann Dubois and Xuechen Li and Carlos Guestrin and Percy Liang and Tatsunori B. Hashimoto },
  title = {Stanford Alpaca: An Instruction-following LLaMA model},
  year = {2023},
  publisher = {GitHub},
  journal = {GitHub repository},
  howpublished = {\url{https://github.com/tatsu-lab/stanford_alpaca}},
}

@misc{mazeika2024harmbenchstandardizedevaluationframework,
      title={HarmBench: A Standardized Evaluation Framework for Automated Red Teaming and Robust Refusal}, 
      author={Mantas Mazeika and Long Phan and Xuwang Yin and Andy Zou and Zifan Wang and Norman Mu and Elham Sakhaee and Nathaniel Li and Steven Basart and Bo Li and David Forsyth and Dan Hendrycks},
      year={2024},
      eprint={2402.04249},
      archivePrefix={arXiv},
      primaryClass={cs.LG},
      url={https://arxiv.org/abs/2402.04249}, 
}

@inproceedings{
andriushchenko2025jailbreaking,
title={Jailbreaking Leading Safety-Aligned {LLM}s with Simple Adaptive Attacks},
author={Maksym Andriushchenko and Francesco Croce and Nicolas Flammarion},
booktitle={The Thirteenth International Conference on Learning Representations},
year={2025},
url={https://openreview.net/forum?id=hXA8wqRdyV}
}

@inproceedings{chao2025jailbreaking,
  title={Jailbreaking black box large language models in twenty queries},
  author={Chao, Patrick and Robey, Alexander and Dobriban, Edgar and Hassani, Hamed and Pappas, George J and Wong, Eric},
  booktitle={2025 IEEE Conference on Secure and Trustworthy Machine Learning (SaTML)},
  pages={23--42},
  year={2025},
  organization={IEEE}
}

@misc{zou2023universaltransferableadversarialattacks,
      title={Universal and Transferable Adversarial Attacks on Aligned Language Models}, 
      author={Andy Zou and Zifan Wang and Nicholas Carlini and Milad Nasr and J. Zico Kolter and Matt Fredrikson},
      year={2023},
      eprint={2307.15043},
      archivePrefix={arXiv},
      primaryClass={cs.CL},
      url={https://arxiv.org/abs/2307.15043}, 
}

@inproceedings{
liu2024autodan,
title={Auto{DAN}: Generating Stealthy Jailbreak Prompts on Aligned Large Language Models},
author={Xiaogeng Liu and Nan Xu and Muhao Chen and Chaowei Xiao},
booktitle={The Twelfth International Conference on Learning Representations},
year={2024},
url={https://openreview.net/forum?id=7Jwpw4qKkb}
}

@misc{mehrotra2024treeattacksjailbreakingblackbox,
      title={Tree of Attacks: Jailbreaking Black-Box LLMs Automatically}, 
      author={Anay Mehrotra and Manolis Zampetakis and Paul Kassianik and Blaine Nelson and Hyrum Anderson and Yaron Singer and Amin Karbasi},
      year={2024},
      eprint={2312.02119},
      archivePrefix={arXiv},
      primaryClass={cs.LG},
      url={https://arxiv.org/abs/2312.02119}, 
}

@misc{chao2024jailbreakingblackboxlarge,
      title={Jailbreaking Black Box Large Language Models in Twenty Queries}, 
      author={Patrick Chao and Alexander Robey and Edgar Dobriban and Hamed Hassani and George J. Pappas and Eric Wong},
      year={2024},
      eprint={2310.08419},
      archivePrefix={arXiv},
      primaryClass={cs.LG},
      url={https://arxiv.org/abs/2310.08419}, 
}

@misc{grattafiori2024llama3herdmodels,
      title={The Llama 3 Herd of Models}, 
      author={Aaron Grattafiori and Abhimanyu Dubey and Abhinav Jauhri and Abhinav Pandey and Abhishek Kadian and Ahmad Al-Dahle and Aiesha Letman and Akhil Mathur and Alan Schelten and Alex Vaughan and Amy Yang and Angela Fan and others},
      year={2024},
      eprint={2407.21783},
      archivePrefix={arXiv},
      primaryClass={cs.AI},
      url={https://arxiv.org/abs/2407.21783}, 
}

@misc{gemmateam2024gemma2improvingopen,
      title={Gemma 2: Improving Open Language Models at a Practical Size}, 
      author={{Gemma Team, Google}},
      year={2024},
      eprint={2408.00118},
      archivePrefix={arXiv},
      primaryClass={cs.CL},
      url={https://arxiv.org/abs/2408.00118}, 
}

@misc{touvron2023llama2openfoundation,
      title={Llama 2: Open Foundation and Fine-Tuned Chat Models}, 
      author={Hugo Touvron and Louis Martin and Kevin Stone and Peter Albert and Amjad Almahairi and Yasmine Babaei and Nikolay Bashlykov and Soumya Batra and others},
      year={2023},
      eprint={2307.09288},
      archivePrefix={arXiv},
      primaryClass={cs.CL},
      url={https://arxiv.org/abs/2307.09288}, 
}

@inproceedings{
zhou2023lima,
title={{LIMA}: Less Is More for Alignment},
author={Chunting Zhou and Pengfei Liu and Puxin Xu and Srini Iyer and Jiao Sun and Yuning Mao and Xuezhe Ma and Avia Efrat and Ping Yu and LILI YU and Susan Zhang and Gargi Ghosh and Mike Lewis and Luke Zettlemoyer and Omer Levy},
booktitle={Thirty-seventh Conference on Neural Information Processing Systems},
year={2023},
url={https://openreview.net/forum?id=KBMOKmX2he}
}

@misc{arditi2024refusallanguagemodelsmediated,
      title={Refusal in Language Models Is Mediated by a Single Direction}, 
      author={Andy Arditi and Oscar Obeso and Aaquib Syed and Daniel Paleka and Nina Panickssery and Wes Gurnee and Neel Nanda},
      year={2024},
      eprint={2406.11717},
      archivePrefix={arXiv},
      primaryClass={cs.LG},
      url={https://arxiv.org/abs/2406.11717}, 
}

@misc{wollschlager2026geometryrefusallargelanguage,
      title={The Geometry of Refusal in Large Language Models: Concept Cones and Representational Independence}, 
      author={Tom Wollschläger and Jannes Elstner and Simon Geisler and Vincent Cohen-Addad and Stephan Günnemann and Johannes Gasteiger},
      year={2026},
      eprint={2502.17420},
      archivePrefix={arXiv},
      primaryClass={cs.LG},
      url={https://arxiv.org/abs/2502.17420}, 
}

@inproceedings{
lin2024the,
title={The Unlocking Spell on Base {LLM}s:  Rethinking Alignment via In-Context Learning},
author={Bill Yuchen Lin and Abhilasha Ravichander and Ximing Lu and Nouha Dziri and Melanie Sclar and Khyathi Chandu and Chandra Bhagavatula and Yejin Choi},
booktitle={The Twelfth International Conference on Learning Representations},
year={2024},
url={https://openreview.net/forum?id=wxJ0eXwwda}
}

@ARTICLE{shannon1951english,
  author={Shannon, C. E.},
  journal={The Bell System Technical Journal}, 
  title={Prediction and entropy of printed English}, 
  year={1951},
  volume={30},
  number={1},
  pages={50-64},
  doi={10.1002/j.1538-7305.1951.tb01366.x}}

@inproceedings{raychev2014code,
author = {Raychev, Veselin and Vechev, Martin and Yahav, Eran},
title = {Code completion with statistical language models},
year = {2014},
isbn = {9781450327848},
publisher = {Association for Computing Machinery},
address = {New York, NY, USA},
url = {https://doi.org/10.1145/2594291.2594321},
doi = {10.1145/2594291.2594321},
booktitle = {Proceedings of the 35th ACM SIGPLAN Conference on Programming Language Design and Implementation},
pages = {419–428},
numpages = {10},
location = {Edinburgh, United Kingdom},
series = {PLDI '14}
}

@article{piantadosi2014zipf,
  title={Zipf’s word frequency law in natural language: A critical review and future directions},
  author={Piantadosi, Steven T},
  journal={Psychonomic bulletin \& review},
  volume={21},
  number={5},
  pages={1112--1130},
  year={2014},
  publisher={Springer}
}
\bibliographystyle{plain}

\vspace{2cm}
\appendix

\section*{Appendix}
\vspace{0.5cm}
\section{Related Works}
\label{app:related_works}
\paragraph{Shallow (safety) alignment.} Qi et al.~\citep{qi2025safety} identified a shortcut in current safety alignment procedures: apparent refusal behavior can often be induced by modifying the model's generative distribution primarily over the first few response tokens. They argued that this shallow alignment contributes to a range of downstream vulnerabilities, and thus motivating that ``safety alignment should be made more than a few token''. This phenomenon is closely related to the Superficial Alignment Hypothesis~(SAH)~\citep{zhou2023lima}, which argues that alignment in current LLMs may largely change the surface form of model-user interaction rather than deeply altering the model's underlying capabilities or behavior. In addition, Lin et al.~\citep{lin2024the} showed that the differences introduced by alignment fine-tuning between aligned and unaligned base models diminish as the generated sequence becomes longer. Our work builds on this line of research but asks a different question. Rather than characterizing the existence of shallow alignment, we analyze the learning dynamics of safety alignment to understand why such shallowness naturally arises under autoregressive generation and what broader failure mechanism it reflects. We show that autoregressive consistency provides a mechanistic account of shallow safety alignment and further suggests vulnerabilities beyond prefix-targeting attacks. A deeper connection between our analysis and the broader phenomenon of superficial alignment is left to future work.

\paragraph{Understanding safety of LLMs.}A recent line of work studies LLM safety by probing and intervening on internal model states. Arditi et al.~\citep{arditi2024refusallanguagemodelsmediated} showed that refusal behavior can be mediated by a one-dimensional subspace, while Wollschlager et al.~\citep{wollschlager2026geometryrefusallargelanguage} argued that refusal is better characterized by multiple independent directions and higher-dimensional concept cones, pointing to a richer geometric structure underlying refusal. More recently, Zhao et al.~\citep{zhao2026llms} distinguished harmfulness recognition from refusal behavior, showing that harmfulness is represented at the user-instruction position whereas refusal is represented at the post-instruction position. Our work is complementary: we show that even after refusal has been established in the output trajectory, generation can still be redirected by a harmful continuation state. This suggests that safety alignment should not merely elicit refusal, but should connect harmfulness recognition to stable recovery throughout generation. In addition, Xu et al.~\citep{xu2024safedecodingdefendingjailbreakattacks} observed that safety-disclaimer tokens often remain among the top-ranked next-token candidates under jailbreak attacks, and used this signal to amplify safety tokens during decoding. However, for efficiency, their defense protects only the first few tokens. In contrast, our results show that this can be insufficient, since harmful autoregressive consistency may be induced later inside the generation trajectory.

\paragraph{Defenses for LLMs.} A family of defenses inspired by adversarial training (AT) has been developed to improve LLM robustness, typically by augmenting training with adversarial prompts or perturbations generated dynamically during training. Casper et al.~\citep{casper2025defendingunforeseenfailuremodes} studied latent adversarial training (LAT), which applies adversarial perturbations to hidden representations, motivated by the view that latent states encode more compressed and abstract features used by the model. Xhonneux et al.~\citep{xhonneux2024efficientadversarialtrainingllms} proposed continuous adversarial training in token-embedding space, including CAT and CAPO, and showed that robustness to continuous embedding perturbations can transfer to discrete jailbreak attacks such as GCG, AutoDAN, and PAIR. Sheshadri et al.~\citep{sheshadri2025latentadversarialtrainingimproves} further extended LAT from untargeted perturbations to targeted latent perturbations that explicitly steer the model toward undesirable behaviors. While these methods primarily define adversarial perturbations in the input or latent space, our work offers a complementary trajectory-level perspective: the adversary should also be defined over autoregressive states induced during generation, where harmful continuation can be preserved and extended.

\section{Proofs for Section~\ref{sec:theory}}
\label{app:proofs}
\subsection{Proof of Lemma \ref{lemma:Jacobian-bound}}
\label{app:proof-Jacobian-bound}
\begin{proof}
    By a straightforward computation, the softmax Jacobian
 \[A_t = \mathrm{diag}(p_\mathrm{\theta}(\cdot \,|\sqx, \sqy_{:t} )) - p_\mathrm{\theta}(\cdot \,|\sqx, \sqy_{:t} ) p_\mathrm{\theta}(\cdot \,|\sqx, \sqy_{:t} )^T,\] 
 where $ p_\mathrm{\theta}(\cdot \,|\sqx, \sqy_{:t} )$ is viewed as a column vector whose $i$th entry is $ p_\mathrm{\theta}(i\mathrm{th \ token} \,|\sqx, \sqy_{:t} )$, and $\mathrm{diag}(u)$ forms a diagonal matrix from a vector $u$ where the diagonal entries are the components of $u$. We first prove that all eigenvalues of $A_t$ are between 0 and 1. Consider an arbitrary normalized vector $v$, we have 
  \[v^T A_t v = \sum_i p_i v_i^2 - \left(\sum_i p_i v_i\right)^2, 
\]
since $p_i$ are probabilities, this is just the variance of $v_i$ under probability distribution $p_i$, and therefore always non-negative. This shows that $A_t$ is positive semi-definite. 
Furthermore, $v^T A_t v < \sum_i p_i v_i^2$, and since $v$ is normalized such that $v_i^2 \leq 1$, we have  $v^T A_t v < \sum_i p_i =1$, which proves that all eigenvalues of $A_t$ are smaller than 1.

To prove Lemma \ref{lemma:Jacobian-bound}, we note 
   \[ \|A_t\|_F^2 \; = \mathrm{tr}(A_t^2)
    \;\leq\;\mathrm{tr}(A_t)
    \;=\; (1-\|p_\mathrm{\theta}(\cdot \,|\sqx, \sqy_{:t} )\|^2)  
    \;\leq\; 1 - p_\mathrm{\theta}(y_{t+1} \,|\sqx, \sqy_{:t} )^2,\]
At initialization $\theta = \mathrm{\theta_{base}}$, Assumption \ref{assmp:consistency_alignment} shows that $p_\mathrm{\theta}(y_{t+1} \,|\sqx, \sqy_{:t} ) > 1-\epsilon$ for all $t>t_c$. Therefore 
\[
\|A_t\|_F^2 < 2\epsilon
\]
for all $t>t_c$ at $\theta = \mathrm{\theta_{base}}$.

\end{proof}
\subsection{Proof of Proposition~\ref{prop:loss-func-grad}}
\label{app:proof-loss-func-grad}
\begin{proof}
For supervised fine tuning, 
\[L_\mathrm{SFT} =  - \sum_t\E_{(\sqx, \sqr)\sim \gD} \left[\log  p_\mathrm{\theta}(r_{t+1}|\sqx, \sqr_{:t})\right]
\]
where $\sqx$ is a harmful prompt and $\sqr$ is a refusal response.
\[\nabla_\theta L_\mathrm{SFT} = - \nabla_\theta\sum_{t\leq t_c}\E_{(\sqx, \sqr)\sim \gD} \left[\log  p_\mathrm{\theta}(r_{t+1}|\sqx, \sqr_{:t})\right] -  \sum_{t> t_c}\E_{(\sqx, \sqr)\sim \gD} \left[\frac{\nabla_\theta p_\mathrm{\theta}(r_{t+1}|\sqx, \sqr_{:t})}{p_\mathrm{\theta}(r_{t+1}|\sqx, \sqr_{:t})}\right] 
\]
By assumption $p_\mathrm{\theta_{base}}$ is autoregressive on $\mathcal D$, therefore by proposition \ref{prop:gradient_concentration} 
\[
\left|\left|\sum_{t> t_c}\E_{(\sqx, \sqr)\sim \gD} \left[\frac{\nabla_\theta p_\mathrm{\theta}(r_{t+1}|\sqx, \sqr_{:t})}{p_\mathrm{\theta}(r_{t+1}|\sqx, \sqr_{:t})}\right] \right|\right|_{\theta = \theta_\mathrm{base}} < \frac{(2\epsilon)^\frac{1}{2}}{1-\epsilon}(T-t_c),
\]
where $T$ is the maximal length of the refusal sequence.

For direct preference optimization, 
\begin{equation}
    L_\mathrm{DPO} =  - \E_{(\sqx, \sqr, 
    \sqh)\sim \gD} \left[\log  \sigma (\beta k_\theta(\sqx, \sqr,\sqh))\right]
\end{equation}
where $\sigma$ is the logistic function and 
\[k_\theta(\sqx, \sqr,\sqh) = \sum_t\log\frac{p_\theta(r_t|\sqx, \sqr_{:t})}{p_\mathrm{\theta_{base}}(r_t|\sqx, \sqr_{:t})} - \sum_t\log\frac{p_\theta(h_t|\sqx, \sqh_{:t})}{p_\mathrm{\theta_{base}}(h_t|\sqx, \sqh_{:t})}.\]
In the data triplet $(\sqx, \sqr, \sqh)$, $x$ is the harmful prompt, $r$ is a refusal response and $h$ is a harmful response. 
\[
\nabla_\theta L_\mathrm{DPO}
    \;=\;
    -\beta\,\mathbb{E}\!\left[
        [1- \sigma(\beta k_\theta)]\bigl(
            \sum_t \nabla_\theta \log p_\theta(r_{t+1} \mid \sqx, \sqr_{:t})
            -
           \sum_t \nabla_\theta \log p_\theta(h_{t+1} \mid \sqx, \sqh_{:t})
        \bigr)
    \right].
\]
At initialization, $1- \sigma(\beta k_\theta(\sqx, \sqr,\sqh)) =1/2$ which is a constant, therefore we can again split $\nabla_\theta L_\mathrm{DPO}$ into $t\leq t_c$ and $t>t_c$ contributions just as we did for $L_\mathrm{SFT}$.  The $t>t_c$ contribution can be bounded in the same manner as the SFT bound:
\begin{align*}
   & \frac{1}{2}\left|\left|
 \sum_{t>t_c} \nabla_\theta \log p_\theta(r_{t+1} \mid \sqx, \sqr_{:t})
            -
           \sum_{t>t_c} \nabla_\theta \log p_\theta(h_{t+1} \mid \sqx, \sqh_{:t})
        \right|\right|_\mathrm{\theta = \theta_{base}} \\
\leq &\frac{1}{2}\left|\left|
 \sum_{t>t_c} \nabla_\theta \log p_\theta(r_{t+1} \mid \sqx, \sqr_{:t})
         \right|\right|_\mathrm{\theta = \theta_{base}}    + \frac{1}{2}
          \left|\left| \sum_{t>t_c} \nabla_\theta \log p_\theta(h_{t+1} \mid \sqx, \sqh_{:t})
        \right|\right|_\mathrm{\theta = \theta_{base}} \\
         < & \frac{(2\epsilon)^\frac{1}{2}}{1-\epsilon}(T-t_c)
\end{align*}
where $T$ is the maximal length of $\sqr$ and $\sqh$.
\end{proof}

\subsection{One step version}
Before proving the full multi-step result, we first examine the simpler case of a single gradient update from the base model. This one-step setting already captures the central mechanism: under autoregressive consistency, next-token distributions at late positions have small sensitivity to the alignment gradient, so the update mainly affects early response positions while leaving later positions close to the base model. Proposition~\ref{thm:ac_to_shallow_one_step} makes this statement precise.  It also implies that autoregressive consistency is preserved after the update in a weakened form, as shown in Corollary~\ref{coro:small-update}.  This one-step result provides the basic induction step used in the multi-step proof in Appendix~\ref{app:multistep_shallow_update}. Define 
\begin{equation}
    \Delta_{\theta_{\base} \to \theta_{\ald}} p_\mathrm{\theta}(\cdot \,|\sqx, \sqy_{:t}) := p_\mathrm{\theta_{\ald}}(\cdot \,|\sqx, \sqy_{:t}) - p_\mathrm{\theta_{base}}(\cdot \,|\sqx, \sqy_{:t}).
\end{equation}
\begin{prop}[One-step shallow safety alignment]
\label{thm:ac_to_shallow_one_step}
Let $\eta$ be the learning rate and $L$ be the alignment loss function~(SFT or DPO).  Under conditions of Proposition~\ref{prop:gradient_concentration} and that $\|\nabla_\theta z\| <C$,  then the one-step gradient update from the base model 
\begin{equation}
\begin{aligned}
    \|  \Delta_{\theta_{\base} \to \theta_{\ald}} p_\mathrm{\theta}(\cdot \,|\sqx, \sqy_{:t}) \|_F & =  \left\|-\eta \nabla_{\theta}p_\theta(\cdot \,|\sqx, \sqy_{:t}) \cdot \nabla_{\theta} L\right\|_{\theta = \theta_\mathrm{base}} + O(\eta^2)\\
      &<  \eta C(2\epsilon)^\frac{1}{2}\, \|\nabla_{\theta} L\|_{\theta = \theta_\mathrm{base}}  + O(\eta^2) 
\end{aligned} 
\end{equation}
for all $t>t_c$.
\end{prop}

Note that since the Frobenius norm is (the square root of) a sum over squares, this also implies that the above inequality for any given $i$th token in the vocabulary  $ \Delta p_\mathrm{\theta_{base}}(i\mathrm{th\ token} \,|\sqx, \sqy_{:t}) $, which we summarize as 
\begin{corollary}
\label{coro:small-update}
Under the conditions of Theorem~\ref{thm:ac_to_shallow}, 
\begin{equation}
    | \Delta_{\theta_{\base} \to \theta_{\ald}} p_\mathrm{\theta_{base}}(\mathrm{any\ next \ token}\,|\sqx, \sqy_{:t})| < \eta C(2\epsilon)^\frac{1}{2}\, \|\nabla_{\theta} L\|_{\theta = \theta_\mathrm{base}}, \qquad \forall t>t_c,
\end{equation}
and therefore after one gradient update
\begin{equation}
    p_\mathrm{\theta_1}(y_{t+1}|\sqx, \sqy_{:t}) >  1- \epsilon - \eta C(2\epsilon)^\frac{1}{2}\, \|\nabla_{\theta} L\|_{\theta = \theta_\mathrm{base}}, \qquad \forall t>t_c.
\end{equation}
\end{corollary}
Therefore the gradient updates on late-position tokens are suppresed by $\epsilon^{1/2}$.

\subsection{Proof of Theorem~\ref{thm:ac_to_shallow}}
\label{app:multistep_shallow_update}
We prove the multi-step version of the gradient-concentration argument. Recall that Proposition~\ref{prop:loss-func-grad} shows that, at initialization, the dominant contributions to the SFT or DPO gradient come from the early positions $t\le t_c$, because autoregressive consistency suppresses the gradient contribution of late positions $t > t_c$. We now show that the same concentration persists along gradient descent, and hence the resulting alignment remains shallow.

\begin{proof}
Let $\theta_0 = \theta_{\mathrm{base}}$, and consider gradient descent
\begin{equation}
    \theta_{k+1}=\theta_k-\eta_k\nabla L(\theta_k), \qquad k=0,\ldots,K-1,
\end{equation}
where $L$ is either the SFT objective in Eq.~\eqref{eq:sft_obj} or the DPO objective in Eq.~\eqref{eq:dpo_obj}.  
\paragraph{Smoothness assumptions.}We use the following standard smoothness assumptions along the optimization trajectory. This is the usual local smoothness regime in which a Taylor-expansion is meaningful.
\begin{enumerate}
    \item For all contexts $(\sqx, \sqy_{:t})$ with $t>t_c$, assume that the logit Jacobian is uniformly bounded (as in Proposition~\ref{prop:gradient_concentration}): 
    \begin{equation}
        \tag{A.1} \|\nabla_\theta z_\theta(\cdot\mid x,y_{:t})\|_F \le C. \label{eq:a.1}
    \end{equation}
    \item The conditional next-token distribution is twice differentiable along the line segment between consecutive iterates, and that
    \begin{equation}
       \tag{A.2 } \|\nabla_\theta^2 p_\theta(\cdot\mid \sqx, \sqy_{:t})\|_F \le B \label{eq:a.2}
    \end{equation}
    for all $t > t_c$.
    \item The total update length is small enough so that the sequence $\epsilon_k$ (defined below) satisfies 
    \begin{equation}
        \tag{A.3} \forall k \le K: \ \epsilon_k < 1. \label{eq:a.3}
    \end{equation}
    \item The gradient norm is uniformly bounded along the optimization path:
    \begin{equation}
        \|\nabla L(\theta_k)\|\le G, \qquad k = 0, \ldots, K-1. \tag{A.4} \label{eq:a.4}
    \end{equation}
\end{enumerate}
For a constant learning rate $\eta$, we let
\begin{equation}
    \eta \le \frac{ \sqrt{\epsilon_0 + \varepsilon_2} - \sqrt{\epsilon_0} }{ \alpha KG }, \quad \alpha=\max\left\{\frac{C}{\sqrt 2},\sqrt{\frac B2}\right\}, \label{eq:lr_condition}
\end{equation}
where $\bar{\epsilon_2}$ is defined in Definition~\ref{def:shallow_alignment}. For DPO, the same argument is applied to both the refusal trajectory $\sqr$ and the harmful trajectory $h$, since the DPO gradient contains token-level log-probability gradients from both. For SFT, only the refusal trajectory $\sqr$ is needed. To avoid duplicating notation, we write $\sqy$ (as stated in the Preliminaries in Section~\ref{sec:intro}) for whichever trajectory is currently being considered.

By Assumption~\ref{assmp:consistency_alignment}, for every $t>t_c$,
\[
    p_{\theta_0}(y_{t+1}\mid \sqx, \sqy_{:t}) > 1 - \epsilon_0,
\]
where $\epsilon_0 = \epsilon$. We prove by induction that there exists a nondecreasing sequence
\begin{equation}
    \epsilon_0\le \epsilon_1\le \cdots\le \epsilon_K    
\end{equation}
such that
\begin{equation}
    \forall k \le K, \ t >t_c: \quad p_{\theta_k}(y_{t+1}\mid \sqx, \sqy_{:t}) > 1 - \epsilon_k.   
\end{equation}
\paragraph{Step $0$.}The base case $k=0$ is exactly Assumption~\ref{assmp:consistency_alignment}.
\paragraph{Step $k$.}Now suppose the claim holds at step $k$. Fix a late position $t > t_c$.  By the induction hypothesis,
\[
     p_{\theta_{k}}(\cdot\mid \sqx, \sqy_{:t}) > 1 - \epsilon_k.
\]
Recall that the proof of Lemma~\ref{lemma:Jacobian-bound} does not depend on being exactly at $\theta_{\mathrm{base}}$; it only uses the fact that the conditional distribution assigns probability at least $1 - \epsilon_k$ to one token. Therefore, we obtain
\begin{equation*}
    \|\nabla_z p_\theta(\cdot\mid \sqx, \sqy_{:t}) \|_F < \sqrt{2\epsilon_k}.    
\end{equation*}
By the chain rule and the bounded-logit-Jacobian assumption~\eqref{eq:a.1},
\begin{equation}
\begin{aligned}
   \|\nabla_\theta p_{\theta_{k}}(\cdot\mid \sqx, \sqy_{:t})\|_F  & \le   \|\nabla_z p_{\theta_{k}}(\cdot\mid \sqx, \sqy_{:t})\|_F \|\nabla_\theta z_{\theta_k}\|_F\\
    & < C\sqrt{2\epsilon_k}. \label{eq:bound_grad_p}
\end{aligned}
\end{equation}
Let
\begin{equation*}
    \Delta_k := \theta_{k+1} - \theta_k = -\eta_k \nabla L(\theta_k).
\end{equation*}
By Taylor expansion and the second-order smoothness assumption~\eqref{eq:a.2},
\begin{equation}
\begin{aligned}
     \|p_{\theta_{k + 1}}(\cdot\mid \sqx, \sqy_{:t}) - p_{\theta_{k}}(\cdot\mid \sqx, \sqy_{:t}) \|_F & \le \|\nabla_\theta p_{\theta_{k}}(\cdot\mid \sqx, \sqy_{:t})\|_F \|\Delta_k\| + \frac{B}{2}\|\Delta_k\|^2 \\
     & < C \sqrt{2\epsilon_k}\eta_k\|\nabla L(\theta_k)\| + \frac{B}{2}\eta_k^2\|\nabla L(\theta_k)\|^2,
\end{aligned}
\end{equation}
where we substitute the bound from Eq.~\eqref{eq:bound_grad_p} in the second inequality.

Now define
\begin{equation*}
    \epsilon_{k+1} := \epsilon_k + C\sqrt{2\epsilon_k}\eta_k\|\nabla L(\theta_k)\| + \frac{B}{2}\eta_k^2\|\nabla L(\theta_k)\|^2.
\end{equation*}
Then
\begin{equation}
\begin{aligned}
    p_{\theta_{k + 1}}(\cdot\mid \sqx, \sqy_{:t}) & \ge p_{\theta_{k}}(\cdot \mid \sqx, \sqy_{:t})
    - \|p_{\theta_{k + 1}}(\cdot\mid \sqx, \sqy_{:t}) - p_{\theta_{k}}(\cdot\mid \sqx, \sqy_{:t})\|_F \\
    & > 1 - \epsilon_k - (\epsilon_{k+1} - \epsilon_k) \\
    & = 1-\epsilon_{k+1}.
\end{aligned}
\end{equation}
This completes the induction. Hence autoregressive consistency is preserved at every late position, with a gradually weakened constant $\epsilon_k$.

We next show that this implies gradient concentration throughout training, not merely at initialization.
\paragraph{Gradient concentration for SFT.}For SFT~(Eq.~\eqref{eq:sft_obj}) and $t > t_c$, using the induction result above $p_{\theta_k}(r_{t+1}\mid \sqx, \sqr_{:t}) > 1 - \epsilon_k,$ we obtain that (similar to the proof in Appendix~\ref{app:proof-loss-func-grad})
\begin{equation}
\begin{aligned}
    \left\|
        - \nabla_\theta \log p_{\theta_k}(r_{t+1}\mid \sqx, \sqr_{:t})
    \right\|_F & =
    \left\|
        \frac{
            \nabla_\theta p_{\theta_k}(r_{t+1} \mid \sqx, \sqr_{:t})
        }{
        p_{\theta_k}( r_{t+1} \mid \sqx, \sqr_{:t})
        }
    \right\|_F \\
    & \le  \frac{\|\nabla_\theta p_{\theta_k}(\cdot\mid \sqx, \sqr_{:t})\|_F}{1 - \epsilon_k}\\
    & < \frac{C\sqrt{2\epsilon_k}}{1-\epsilon_k}.
\end{aligned}    
\end{equation}
Thus every late SFT token-gradient contribution is $O(\sqrt{\epsilon_k})$. Summing over $t>t_c$, the total late-position contribution to the SFT gradient is bounded by
\begin{equation}
    \left\| 
        \sum_{t>t_c} - \nabla_\theta \log p_{\theta_k}(r_{t+1}\mid \sqx, \sqr_{:t})
    \right\|_F \le (T - t_c) \frac{C\sqrt{2\epsilon_k}}{1-\epsilon_k},
\end{equation}
where $T $ is the response length. Hence, as long as $\epsilon_k$ remains small, the SFT gradient contribution from $t > t_c$ remains suppressed at every step $k$. This is exactly the multi-step analogue of Proposition~\ref{prop:loss-func-grad} for SFT.
\paragraph{DPO.}For DPO, recall that the DPO objective is
\[
    L_{\mathrm{DPO}}(\theta)
    =
    -\mathbb E_{(x,r,h)\sim D}
    \left[
    \log \sigma\left(\beta k_{\theta}(x,r,h)\right)
    \right],  
\]

where
\[
    k_{\theta}(\sqx, \sqr, \sqh) = \log \frac{p_\theta(\sqr \mid \sqx)}{p_{\theta_{\mathrm{base}}}(\sqr \mid \sqx)} - \log \frac{p_\theta(\sqh\mid \sqx)}{p_{\theta_{\mathrm{base}}}(\sqh \mid \sqx)} .
\]
Since the base-model terms do not depend on $\theta$, we can further expand its gradient as 
\[    
    \nabla_\theta L_{\mathrm{DPO}}(\theta) = -\beta
    \mathbb E_{(\sqx, \sqr, \sqh)\sim \gD} \left[ \sigma(-\beta k_{\theta}(\sqx, \sqr, \sqh)) \nabla_\theta k_{\theta} (\sqx, \sqr, \sqh) \right],
\]
with $ 0<\sigma(-\beta k_{\theta}(\sqx, \sqr, \sqh))<1$. Furthermore,
\begin{equation}
\begin{aligned}
    \nabla_\theta k_{\theta}(\sqx, \sqr, \sqh) & = \nabla_\theta \log p_\theta(\sqr \mid \sqx) - \nabla_\theta \log p_\theta(\sqh \mid \sqx) \\
    & = \sum_t \left[ \nabla_\theta \log p_\theta(r_{t+1} \mid \sqx, \sqr_{:t}) - \nabla_\theta \log p_\theta(h_{t+1} \mid \sqx, \sqh_{:t})\right],
\end{aligned}
\end{equation}
where we use the autoregressive factorization in the second equality. Therefore, the DPO gradient is a bounded scalar reweighting of the difference between the refusal and harmful token-level log-probability gradients. Similar to the case of SFT, for every late refusal token $t > t_c$ and every late harmful token, the induction result above gives, $\forall t > t_c$:
\begin{equation}
\begin{aligned}
    & p_{\theta_k}(r_{t+1} \mid \sqx, \sqr_{:t})>1-\epsilon_k \overset{\text{suppression at late tokens}}{\implies} \left\| \nabla_\theta \log p_{\theta_k}(r_{t+1}\mid \sqx, \sqr_{:t}) \right\|_F < \frac{C\sqrt{2\epsilon_k}}{1-\epsilon_k},\\
    & p_{\theta_k}(h_{t+1}\mid \sqx, \sqh_{:t}) > 1 - \epsilon_k   \overset{\text{suppression at late tokens}}{\implies}  \left\| \nabla_\theta \log p_{\theta_k}(h_{t+1}\mid \sqx, \sqh_{:t}) \right\|_F < \frac{C\sqrt{2\epsilon_k}}{1-\epsilon_k}.
\end{aligned}
\end{equation}
Consequently, the late-position contribution to the DPO gradient satisfies
\begin{equation}
\begin{aligned}
   &\  \|\nabla_\theta L_{\mathrm{DPO},\,t>t_c}(\theta_k)\|_F \\
   \le &\  \beta\, \mathbb E_{(\sqx, \sqr, \sqh)\sim \gD} \left[ \sum_{t>t_c} \left\| \nabla_\theta \log p_{\theta_k}(r_{t+1}\mid \sqx, \sqr_{:t}) \right\|_F +  \sum_{t>t_c} \left\|  \nabla_\theta \log p_{\theta_k}(h_{t+1}\mid \sqx, \sqh_{:t}) \right\|_F \right] \\
    < & \ 2\beta (T - t_c) \frac{C\sqrt{2\epsilon_k}}{1-\epsilon_k},
\end{aligned}
\end{equation}
where $T$ upper-bounds the response lengths. Hence the  DPO gradient at late positions is also $O(\sqrt{\epsilon_k})$ at every training step $k \leq K$. This gives the multi-step analogue of Proposition~\ref{prop:loss-func-grad} for DPO.

Therefore, for both SFT and DPO, the gradient contributions at late positions remain suppressed throughout training. The only token positions that can provide dominate learning signal are therefore the early positions $t \le t_c$. In this precise sense, the alignment gradient remains concentrated on the early tokens throughout gradient descent. It remains to prove that the final aligned model stays close to the base model at late positions. From the Taylor bound above,
$
\|p_{\theta_{k+1}}(\cdot \mid \sqx, \sqy_{:t}) - p_{\theta_{k }}(\cdot \mid \sqx, \sqy_{:t})\|_F \le \epsilon_{k+1}-\epsilon_k .$ An easy summation over $k=0, \ldots, K-1$ gives
\begin{equation}
\begin{aligned}
    \|p_{\theta_K}(\cdot\mid \sqx, \sqy_{:t}) - p_{\theta_0}(\cdot\mid \sqx, \sqy_{:t})\|_F 
    & \le \sum_{k=0}^{K-1}  \|p_{\theta_{k+1}}(\cdot\mid \sqx, \sqy_{:t}) -  p_{\theta_k}(\cdot\mid \sqx, \sqy_{:t})\|_F \\
    & \le \sum_{k=0}^{K-1}  (\epsilon_{k+1} - \epsilon_k) \\
    & = \epsilon_K - \epsilon_0 \\
\end{aligned}
\end{equation}
\begin{equation}
    \implies \|p_{\theta_K}(\cdot\mid \sqx, \sqy_{:t}) - p_{\theta_{\base}}(\cdot\mid \sqx, \sqy_{:t})\|_F \le \epsilon_K-\epsilon_0.
\end{equation}
It remains to verify that $\epsilon_K - \epsilon_0$ is smaller than the tolerance $\varepsilon_2$ required in Definition~\ref{def:shallow_alignment}. Let $S_K=\sum_{k=0}^{K-1}\eta_k\|\nabla L(\theta_k)\|.$
From the recursion defining $\epsilon_k,$ we have
\[
    \sqrt{\epsilon_K}\le \sqrt{\epsilon_0}+\alpha S_K.
\]
Under the bounded-gradient assumption~\eqref{eq:a.4}, $S_K \le KG \eta.$  Therefore, $\epsilon_K-\epsilon_0 \le 2\alpha KG\eta\sqrt{\epsilon_0} + \alpha^2K^2G^2\eta^2.$ By the condition on learning rate given in Eq.~\eqref{eq:lr_condition}, this quantity is at most  $\varepsilon_2$. Hence the second condition in Definition~\ref{def:shallow_alignment} is established.

The first condition in Definition~\ref{def:shallow_alignment} concerns the model being aligned at early positions. Since the effective learning signal of SFT or DPO is concentrated on the early positions $t\le t_c$, under the usual assumption that the alignment optimization succeeds on the part of the objective where the gradient is not suppressed, the trained model satisfies the this condition as a given. Therefore, the alignment is shallow.
\end{proof}

\section{Experimental Details}
\label{app:exp}
\begin{algorithm}[h]
\caption{Adversarial Safety Alignment with Worst-Insertion Attack}
\label{alg:asa-worst-insertion}
\begin{algorithmic}[1]
\Require Model $p_\theta$; safety triplet dataset $\mathcal{D}_{\mathrm{Safety}}$
containing $(\sqx,\sqr,\sqh)$; utility dataset $\mathcal{D}_{\mathrm{Utility}}$;
harmful span length $q$; number of sampled insertion positions $k$;
refusal restart bank $\mathcal{R}=\{r^C_j\}_{j=1}^m$;
utility weight $\lambda \in (0,1]$
\Ensure Safety-aligned model parameters $\theta$

\For{each training step}
    \State Sample a safety minibatch $\mathcal{B}_S \sim \mathcal{D}_{\mathrm{Safety}}$
    \State Sample a utility minibatch $\mathcal{B}_U \sim \mathcal{D}_{\mathrm{Utility}}$
    \State Initialize $L_{\mathrm{Insertion}} \gets 0$

    \For{each triplet $(\sqx, \sqr, \sqh) \in \mathcal{B}_S$}
        \State Take the harmful prefix span $\sqh_{:q}$
        \State Uniformly sample insertion positions
        $\{i_j\}_{j=1}^k \sim \mathrm{Unif}([|\sqr|])$
        \State Construct candidate perturbed states $\mathcal{A}(\sqx,\sqr,\sqh)
            =
            \left\{
            \tau_j = \bigl(\sqx,\,[\sqr_{:i_j};\sqh_{:q}]\bigr)
            \right\}_{j=1}^k$
        \For{each candidate state $\tau_j \in \mathcal{A}(\sqx, \sqr, \sqh)$}
            \State Compute harmful-continuation score $H(\tau_j;\theta)
                =
                \frac{1}{T_o}
                \log p_\theta
                \left(
                \sqh_{q+1:q+T_o}
                \mid
                \tau_j
                \right)$
            \State Compute refusal-restart score $R(\tau_j;\theta)
                =
                \max_{\sqr^C \in \mathcal{R}}
                \frac{1}{|\sqr^C|}
                \log p_\theta
                \left(
                \sqr^C
                \mid
                \tau_j
                \right)$
            \State Compute harmful autoregressive consistency margin $\Phi_\theta(\tau_j)
                =
                H(\tau_j;\theta)-R(\tau_j;\theta)$
        \EndFor
        \State Select the worst insertion state
        \[
            \tau^\star_\theta(\sqx, \sqr, \sqh)
            \in
            \arg\max_{\tau_j \in \mathcal{A}(\sqx, \sqr, \sqh)}
            \Phi_\theta(\tau_j)
        \]
        \State Accumulate refusal-recovery loss $L_{\mathrm{Insertion}}
            \gets
            L_{\mathrm{Insertion}}
            -
            \log p_\theta
            \left(
            \sqr
            \mid
            \tau^\star_\theta(\sqx, \sqr, \sqh)
            \right)$
    \EndFor
    \State Compute utility loss $L_{\mathrm{Utility}}
        =
        -
        \frac{1}{|\mathcal{B}_U|}
        \sum_{(\sqz,\sqy)\in \mathcal{B}_U}
        \log p_\theta(\sqy\mid \sqz)$
    \State Form total objective $L
        =
        \lambda L_{\mathrm{Insertion}}
        +
        (1-\lambda)L_{\mathrm{Utility}}$
    \State Update parameters $\theta$ by gradient descent on $L$
\EndFor
\State \Return $\theta$
\end{algorithmic}
\label{alg:aa}
\end{algorithm}

For convenience, we first present the adversarial safety alignment with random-worst insertion in Algorithm~\ref{alg:aa}.

\paragraph{Training configurations.}In all fine-tuning experiments which involve both $D_{\safety}$ and $D_{\uti}$, including those in Tab.~\ref{tab:asr_comparison}, and Tab.~\ref{tab:asr_comparison_more_models}, we use configurations in Tab.~\ref{tab:training}, where the value of $\lambda$ is controlled by batch size.
\begin{table}[h!]
    \centering
    \begin{tabular}{l|c}
    \toprule

     Batch size from $D_{\safety}$   &  16\\
    Batch size from $D_{\uti}$     &  16\\
     Optimizer & AdamW with default configurations\\
     Learning rate & $2 \times 10^{-5}$\\
     Epoch & 10\\
     Training attack span & 50 tokens\\
     \bottomrule
    \end{tabular}
    \vspace{0.3cm}
    \caption{Training configurations for Tab.~\ref{tab:asr_comparison} and \ref{tab:asr_comparison_more_models}}
    \label{tab:training}
\end{table}
\paragraph{Evaluation of ASR.}In this paper, attacks construct a partial perturbed autoregressive state $\tilde{\tau}_{t}$ from the data $(\sqx, \sqr, \sqh)$, where $\sqx$ is the harmful prompt, $\sqr$ is the safe refusal response, and $\sqh$ is the harmful response (could be a short span). The model then continues generation from this perturbed state according to $p_{\theta}(\cdot \mid \tilde{\tau}_t)$ at position $t$. To conveniently evaluate the ASR under an attack, we first construct a set of refusal start keywords $\gC$ that commonly appear in a safe response, and let the model continue the generation up to 50 tokens by $p_{\theta}(\sqy \mid \tilde{\tau}_t)$, where $|\sqy| = 50$. Then the attack is considered successful if none of the keywords in \(\gC\) appears in \(\sqy\).

\paragraph{Dataset.}(1) For training data, as indicated in Section~\ref{sec:asa_framework}, we use a triple dataset $(\sqx, \sqr, \sqh)\sim D_{\safety}$ for both prefill and random insertion attack. We use the dataset open-sourced by Qi et al.~\citep{qi2025safety}. Specifically, this dataset consists of 256 harmful instructions $\sqx$, primarily from Ganguli et al.~\citep{ganguli2022redteaminglanguagemodels}, which does not overlap with any safety evaluation datasets used in this paper. (2) The utility data $D_{\uti}$ is obtained from the benign instructions from Alpaca~\citep{alpaca} paired with a response from the initial \texttt{Llama-2-7B-Chat}. This dataset aims to make the fine-tuned model preserve its original behavior on benign instructions. (3) For evaluation data, we use the HEx-PHI safety benchmark~\citep{qi2025safety} for experiments in Fig.~\ref{fig:random_attack_on_raw_models}, Tab.~\ref{tab:asr_comparison}, and Tab.~\ref{tab:asr_comparison_more_models}. This dataset contains 330 harmful instructions spanning 11 categories, and we evaluate whether the model produces compliant responses to these harmful requests.

\paragraph{Candidate refusal bank $\gR$.}
For adversarial safety alignment with random worst-insertion, we use
$\gR = \{\text{``I cannot fulfill'',\, ``I apologize, but I cannot'',\, ``I am unable'',\, ``I'm sorry''}\}$. We intentionally keep this bank small to avoid introducing significant additional computational overhead.

\paragraph{Additional details for experiments in Tab.~\ref{tab:other_attacks}.}To evaluate ASR against these attacks, we use the official HarmBench codebase~\citep{mazeika2024harmbenchstandardizedevaluationframework}, which provides both the attack optimization procedures and the final evaluation protocol. We keep all settings at their default values and evaluate on the HarmBench dataset. Since evaluating optimization-based attacks is computationally expensive, we use the subset of 200 standard harmful behaviors. For the aligned model, we randomly select one of the three models trained with adversarial safety alignment using random worst-insertion from the experiments in Tab.~\ref{tab:asr_comparison}.
\end{document}